\documentclass{article}

\usepackage[utf8]{inputenc} % allow utf-8 input
\usepackage[T1]{fontenc}    % use 8-bit T1 fonts
\usepackage{hyperref}       % hyperlinks
\usepackage{url}            % simple URL typesetting
\usepackage{booktabs}       % professional-quality tables
\usepackage{amsfonts}       % blackboard math symbols
\usepackage{nicefrac}       % compact symbols for 1/2, etc.
\usepackage{tablefootnote}  % for table footnotes
\usepackage{xcolor}     
\usepackage{multirow}
\usepackage{adjustbox}
\usepackage{fullpage}
\usepackage{wrapfig}
\usepackage{makecell}

\usepackage{fullpage}
\usepackage[protrusion=true,expansion=true]{microtype}

\usepackage{algorithm}
\usepackage{algorithmic}
\usepackage{amsmath}
\usepackage{amsthm}
\usepackage{bm}
\usepackage{amssymb}
\usepackage{cleveref}

\usepackage{amsthm}

\newtheorem{remark}{Remark}
\newtheorem{lemma}{Lemma}
\newtheorem{theorem}{Theorem}

\usepackage{colortbl}
\usepackage{multirow}
\usepackage{caption}
\usepackage{subcaption}
\usepackage{times}
\usepackage{natbib}
\usepackage{bm}
\usepackage{pifont}% 

\title{Step-level Denoising-time Diffusion Alignment with Multiple Objectives}

\author{%
    Qi Zhang\thanks{Qi Zhang, Dawei Wang and Shaofeng Zou are with the School of Electrical, Computer and Energy Engineering, Arizona State University, Tempe, AZ 85281 USA (e-mail: \href{mailto:qzhan261@asu.edu}{qzhan261@asu.edu}, \href{mailto:dwang201@asu.edu}{dwang201@asu.edu}, \href{mailto:zou@asu.edu}{zou@asu.edu}). } 
   \quad
   Dawei Wang\footnotemark[1] 
   \quad
 Shaofeng Zou\footnotemark[1] 
}

\begin{document}

\maketitle

\begin{abstract}
Reinforcement learning (RL) has emerged as a powerful tool for aligning diffusion models with human preferences, typically by optimizing a single reward function under a KL regularization constraint. In practice, however, human preferences are inherently pluralistic, and aligned models must balance multiple downstream objectives, such as aesthetic quality and text-image consistency. Existing multi-objective approaches either rely on costly multi-objective RL fine-tuning or on fusing separately aligned models at denoising time, but they generally require access to reward values (or their gradients) and/or introduce approximation error in the resulting denoising objectives. In this paper, we revisit the problem of RL fine-tuning for diffusion models and address the intractability of identifying the optimal policy by introducing a \emph{step-level} RL formulation. Building on this, we further propose \textbf{Multi-objective Step-level Denoising-time Diffusion Alignment} (MSDDA), a retraining-free framework for aligning diffusion models with multiple objectives,  obtaining the optimal reverse denoising distribution in closed form, with mean and variance expressed directly in terms of single-objective base models. We prove that this denoising-time objective is exactly equivalent to the step-level RL fine-tuning, introducing no approximation error. Moreover, we provide numerical results, which indicate our method outperforms existing denoising-time approaches.
\end{abstract}

\section{Introduction}
Diffusion models~\cite{ramesh2022hierarchical,rombach2022high,saharia2022photorealistic} have gained increasing attention in text-to-image generation. However, these models are pre-trained on broad, large-scale datasets and are therefore not tailored to specific downstream tasks. To adapt them to particular application domains, fine-tuning approaches such as supervised fine-tuning~\cite{lee2023aligning} and reinforcement learning (RL) based methods~\cite{black2023training,clark2023directly,fan2023dpok} have been proposed. In RL-based fine-tuning, the goal is to maximize a given reward function by updating a pre-trained model, typically with an additional Kullback–Leibler (KL) regularization term to keep the aligned model close to the pre-trained one.

Despite their success, most RL fine-tuning methods optimize with respect to a \emph{single} reward function. In practice, human preferences are inherently pluralistic; thus, the alignment should balance multiple downstream objectives, such as aesthetic quality and text–image consistency. To address this problem, a multi-objective setting is investigated, in which we have a set of reward functions, and a task-specific reward is given by a preference-weighted combination of these rewards with respect to a weight vector $w$.

Existing work on multi-objective RL can be applied to diffusion fine-tuning. However, these methods require substantial additional computation, such as fine-tuning a large (often exponential in the size of the reward set) number of models to cover the space of preference weights~\cite{rame2023rewarded,zhou2022anchor,yang2019generalized} or solving for conflict-avoiding update directions~\cite{wang2025theoretical}. To improve training efficiency, denoising-time diffusion alignment has been studied, which avoids training new models and instead fuses the denoising processes of existing aligned models to realize a target output distribution. Although these methods avoid extra training, most of them still require access to reward gradients or repeated estimation of value functions by generating many samples with associated rewards~\cite{han2023training,kim2025test,ye2024tfg,singh2025code}. To the best of our knowledge, only a few recent works~\cite{cheng2025diffusion,anonymous2025deradiff} avoid using reward information and obtain a target model by fusing aligned models corresponding to a set of base rewards; however, the derivation of their denoising-time objectives introduces approximation errors that are difficult to quantify. 

In this paper, we aim to address the following question:
\emph{Can we design a retraining-free denoising-time alignment method that does not require access to individual reward functions and introduces no additional approximation error?}

We provide fundamental insights into denoising-time alignment and answer the above question firmly. Our contributions are summarized as follows:
\begin{itemize}
    \item We begin by revisiting existing RL fine-tuning methods for diffusion models, which require sampling from the updated policy. This dependence makes the target policy difficult to track and forces analyses to rely on approximations whose errors are hard to quantify. To overcome this intractability,
    we propose a novel  \emph{step-level} RL fine-tuning formulation. 
    Moreover, we derive the corresponding step-level DPO objective, which trains the model solely from preference pairs and does not require explicit access to the reward function.
    \item Building on our step-level RL formulation, we design a step-level denoising-time diffusion alignment method for multiple objectives. Without any additional training or access to reward functions, our algorithm can, for any preference weight vector $w$, compute the optimal reverse denoising distribution in closed form: both its mean and variance are explicit functions of those of the base reward models. We further show theoretically that the solution obtained from our algorithm is exactly equivalent to that of the step-level RL fine-tuning formulation, and therefore introduces \emph{no} additional approximation error.
    \item We conduct extensive experiments using Stable Diffusion \cite{rombach2022high} as the pre-trained model, considering multiple reward functions and a wide range of preference weights $w$. The results demonstrate that our method outperforms existing denoising-time approaches.
\end{itemize}

\section{Related Work}
\begin{figure*}
    \centering
    \includegraphics[width=0.8\linewidth]{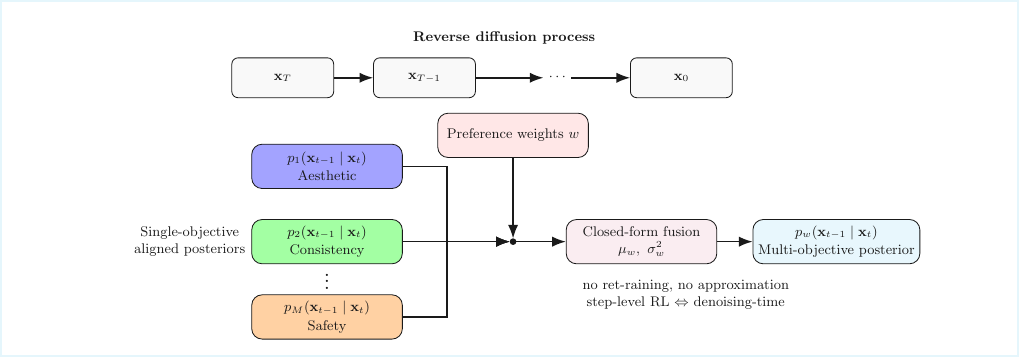}
    \caption{Overview of our proposed Multi-objective Step-level Denoising-time Diffusion Alignment (MSDDA) algorithm.}
    \label{fig:placeholder}
\end{figure*}
\textbf{Single-objective diffusion alignment:} A large number of works study aligning diffusion models to human preferences under a \emph{single} reward. For example, supervised fine-tuning (SFT) methods~\cite{lee2023aligning,wu2023human} improve output quality by minimizing a reward-weighted negative log-likelihood on a fixed dataset. These approaches are entirely offline, with training samples provided in advance rather than generated by the current model.  

RL-based methods such as DDPO~\cite{black2023training} and DPOK~\cite{fan2023dpok} instead formulate the reverse diffusion process as a $T$-horizon Markov Decision Process (MDP), and optimize a terminal-state reward with a KL regularization term to keep the aligned model close to the pre-trained one. Compared with the above methods, DRaFT~\cite{clark2023directly} backpropagates the reward gradient to update the diffusion model. Rather than modifying the diffusion model itself, \cite{hao2023optimizing} uses RL to optimize prompts to improve model performance.  Building on the RL formulation, diffusion DPO~\cite{wallace2024diffusion} can be viewed as a special case: it does not require explicit reward modeling and on-policy sampling, and instead trains the model directly from human preference pairs.

\textbf{Generic multi-objective alignment:} 
While the above works focus on a single objective, a straightforward strategy for multi-objective diffusion alignment is to use linear scalarizations~\cite{roijers2013survey,yang2019generalized,zhou2022anchor,rame2023rewarded}: many models are trained, each corresponding to a distinct preference weight $w$, and for a given user, one selects the model whose weight is closest to the user's preference. To adequately cover the space of preference weights, however, the number of models must be exponential in the number of objectives, making this approach computationally prohibitive. In our framework, the number of aligned models is exactly the same as the number of reward functions, which is much smaller than the ones in these multi-objective RL schemes.

MGDA-based approaches~\cite{desideri2012multiple}  can also be directly applied to SFT~\cite{chen2023three,xiao2023direction,zhang2025mgda} and RL fine-tuning~\cite{wang2025theoretical}. At each update, a conflict-avoiding descent direction that improves all objectives simultaneously is calculated. However, computing this direction introduces nontrivial overhead, and these methods treat diffusion models as general machine learning models, without exploiting the structure of the reverse diffusion process. In contrast, our method only requires single-objective alignment for each base reward, introduces no extra computations such as conflict-avoidant direction solving, and needs no training when adapting to new preference weights at denoising time.

\textbf{LLM decoding-time alignment:}
There have been many recent studies in decoding-time alignment for LLMs~\cite{shi2024decoding,zeng2024token,liu2024decoding}, where multiple alignment policies are fused at inference time to satisfy multi-objective preferences. These methods, however, are designed for auto-regressive token generation and do not directly transfer to the Gaussian-structured denoising dynamics of diffusion models.

\textbf{Diffusion denoising-time alignment:}
\begin{table*}[t]
  \caption{Comparison of existing diffusion denoising-time alignment methods, all of which are training-free.  $\surd$ indicates the term is required and $\times$ indicates the term is not required.}
  \label{tab:comparison}
  \begin{center}
    \begin{small}
      \begin{sc}
        \begin{tabular}{lcccr}
          \toprule
          Method  & Reward Gradient        & Reward Function     & Approximation Error  \\
          \midrule
          MUDM \cite{han2023training}    & $\surd$ &$\surd$ & NA \\
          TFG \cite{ye2024tfg}    & $\surd$ &$\surd$ & NA \\
          DAS \cite{kim2025test}   &  $\times$  &$\surd$ & NA \\
          CoDE\cite{singh2025code}     &  $\times$  &$\surd$&NA \\
          RS\cite{rafailov2023direct}     & $\times$ &$\times$ &NA \\
          DB-MPA \cite{cheng2025diffusion}& $\times$ &$\times$  & $\surd$  \\
          DERADIFF \cite{anonymous2025deradiff}& $\times$ &$\times$ & $\surd$  \\
          \textbf{This paper}   & $\times$ &$\times$  & $\times$ \\
          \bottomrule
        \end{tabular}
      \end{sc}
    \end{small}
  \end{center}
  \vskip -0.1in
\end{table*}
%\textbf{Diffusion denoising-time alignment:}
%\begin{table*}[t]
%  \caption{Comparison of existing diffusion denoising-time alignment methods, all of which are training-free.  $\surd$ indicates the term is required and $\times$ indicates the term is not required.}
%  \label{tab:comparison}
%  \begin{center}
%    \begin{small}
%      \begin{sc}
%        \begin{tabular}{lccccr}
%          \toprule
%          Method  & Reward Gradient        & Reward Function  & RL Formulation    & Approximation Error  \\
%          \midrule
%          MUDM \cite{han2023training}    & $\surd$ &$\surd$& NA & NA \\
%          TFG \cite{ye2024tfg}    & $\surd$ &$\surd$& Eq.~\eqref{eq:rlhf1} & NA \\
%          DAS \cite{kim2025test}   &  $\times$  &$\surd$& Eq.~\eqref{eq:rlhf1} & NA \\
%          CoDE\cite{singh2025code}     &  $\times$  &$\surd$& NA & NA \\
%         RS\cite{rafailov2023direct}     & $\times$ &$\times$ & NA & NA \\
%          DB-MPA \cite{cheng2025diffusion}& $\times$ &$\times$ & Eq.~\eqref{eq:rlhf1} & $\surd$  \\
%          DERADIFF \cite{anonymous2025deradiff}& $\times$ &$\times$ & Eq.~\eqref{eq:rlhf1} & $\surd$  \\
%          \textbf{This paper}   & $\times$ &$\times$ & Eq.~\eqref{eq:rlhf3} & $\times$ \\
%          \bottomrule
%        \end{tabular}
%      \end{sc}
%    \end{small}
%  \end{center}
%  \vskip -0.1in
%\end{table*}
Similar to LLM decoding-time alignment, recent work proposes retraining-free approaches for multi-objective diffusion alignment by modifying the denoising process:

\emph{(1) Reward gradient-based methods.}
For single-objective alignment, several studies~\cite{Yu_2023_ICCV,pmlr-v202-song23k,bansal2023universal,ye2024tfg} use the gradient of a reward function to guide the denoising trajectory toward target regions of the output space. These methods extend naturally to multi-objective settings by using a weighted reward~\cite{han2023training,kim2025test,ye2024tfg}. However, they require differentiable reward models and additional computation to evaluate reward gradients.

\emph{(2) Reward value-based methods.}
To avoid differentiability requirements, other approaches~\cite{kim2025test,li2024derivative,singh2025code} rely on reward values only rather than gradients, adjusting the denoising process based on scalar feedback. While they remove the need for the reward gradient, both value-based and gradient-based methods still require access to the reward function and repeated estimation of reward expectations, typically via Tweedie-based estimators~\cite{kadkhodaie2021stochastic} or Monte Carlo sampling.

\emph{(3) Reward-free fusion methods.}
To eliminate dependence on reward access and additional sample generation, some works require only fusing pre-aligned models. Reward Soup (RS)~\cite{rame2023rewarded} linearly interpolates model parameters rather than their reverse conditional distributions. Although simple, RS requires all aligned models to share the same architecture and lacks theoretical guarantees, e.g., \cite{shi2024decoding} shows that RS fails to generate meaningful responses for LLM tasks.  

Closer to our work, recent reward-free methods such as DB-MPA~\cite{cheng2025diffusion} and DERADIFF~\cite{anonymous2025deradiff} combine the reverse distributions of multiple aligned models. These methods are typically derived from fine-tuning methods whose optimal policy is intractable, and to obtain a tractable denoising-time objective, they omit several terms in the diffusion dynamics, introducing approximation errors that are difficult to quantify. By contrast, our step-level RL fine-tuning formulation resolves this intractability, allowing us to derive an exact closed-form denoising-time objective without dropping any terms. To better illustrate the novelty and advantages of our approach, we provide a detailed comparison with existing denoising-time alignment methods in \Cref{tab:comparison}.

\section{Background}
\subsection{Diffusion model}
Diffusion models \cite{pmlr-v37-sohl-dickstein15, ho2020denoising} are latent variable models that can be expressed by $p_\theta(\mathbf{x}_0)
  := \int p_\theta(\mathbf{x}_{0:T}) \, d\mathbf{x}_{1:T}$, where $\theta$ is the model parameter, $\mathbf{x}_0$ follows the data distribution $\mathbf{x}_0 \sim q(\mathbf{x}_0)$ and $\mathbf{x}_1, \dots, \mathbf{x}_T$ are the latents with the same dimensionality as $\mathbf{x}_0$. 
  There are two processes in the diffusion model: the forward process and the reverse process. The goal is to learn a reverse model $p_\theta(\mathbf{x}_0)$ that approximates the data distribution $q(\mathbf{x}_0)$.
  During each step $t$ of the forward process, Gaussian noise is added with a variance schedule $\beta_t$:
  \begin{flalign}
  q(\mathbf{x}_{1:T}|\mathbf{x}_0)&:=\prod_{t=1}^Tq(\mathbf{x}_{t}|\mathbf{x}_{t-1}), \nonumber
  \\q(\mathbf{x}_{t}|\mathbf{x}_{t-1})&:=\mathcal N(\mathbf{x}_{t}; \sqrt{1-\beta_t}\mathbf{x}_{t-1}, \beta_t \mathbf{I} ),\nonumber
  \end{flalign}
  where $\mathcal N$ denotes Gaussian distribution and $\mathbf{I}$ indicates the identity matrix. In the reverse process, a denoising neural network is trained to remove the noise. Specifically, the reverse starts at $p_\theta(\mathbf{x}_{T})=\mathcal N(\mathbf{x}_{T};\mathbf{0}, \mathbf{I})$ and follows that
  \begin{flalign}
p_\theta(\mathbf{x}_{0:T})&:=p_\theta(\mathbf{x}_{T})\prod_{t=1}^Tp_\theta(\mathbf{x}_{t-1}|\mathbf{x}_{t}), \nonumber\\
     p_\theta(\mathbf{x}_{t-1}|\mathbf{x}_{t})&:=\mathcal N(\mathbf{x}_{t-1}; \mathbf{\mu_\theta}(\mathbf{x}_{t}, t), \mathbf{\Sigma_\theta}(\mathbf{x}_{t}, t)).\nonumber 
  \end{flalign}

\subsection{RL fine-tuning for diffusion}\label{sec:rlhf}
The pre-training of the diffusion model focuses solely on matching the data distribution $q(\mathbf{x}_0)$ and does not explicitly incorporate human preferences. To better align the diffusion model with downstream applications such as aesthetics, text–image alignment, and safety, RL-based fine-tuning is widely leveraged. Concretely, the RL method aims to maximize a given reward function: $\mathbb E_{\mathbf{x}_{0} \sim p_\theta(\mathbf{x}_{0})}[r(\mathbf{x}_{0})]$. However, as noted by \cite{fan2023dpok}, this objective leads to potential over-optimization problems. Following RL fine-tuning methods in LLM tasks \cite{ouyang2022training}, a KL regularization term is added to the RL objective to keep the aligned model close to the pre-trained one:
\begin{flalign}\label{eq:rlhf1}
    \max_{p_\theta} \mathbb E_{\mathbf{x}_{0} \sim p_\theta(\mathbf{x}_{0})}[r(\mathbf{x}_{0})]-\lambda \mathbb D_{\mathrm{KL}}[p_\theta(\mathbf{x}_{0}) ||p_{\mathrm{pre}}(\mathbf{x}_{0})],
\end{flalign}
where $\lambda>0$ is a pre-defined hyper-parameter, $p_{\text{pre}}$ denotes the pre-training model and $\mathbb D_{\mathrm{KL}}$ is the KL divergence. In practice, typically it is hard to track the marginal distribution of $p_\theta(\mathbf{x}_{0})$ and further evaluate the KL term, and then the marginal KL term is replaced by the step-level KL \cite{fan2023dpok,uehara2024understanding}:
\begin{flalign} \label{eq:rlhf2}
    \max_{p_\theta} \mathbb E_{\mathbf{x}_{0:T}\sim p_\theta}\left[r(\mathbf{x}_{0})-\lambda \sum_{t=1}^T\mathbb D_{\mathrm{KL}}[p_\theta^t(\cdot|\mathbf{x}_{t}) ||p_{\mathrm{pre}}^t(\cdot|\mathbf{x}_{t})]\right],
\end{flalign}
where $p^t(\mathbf{x}_{t-1}|\mathbf{x}_{t})$ denotes the reverse conditional distribution. 
%We refer to both Eq.~\eqref{eq:rlhf1} and Eq.~\eqref{eq:rlhf2} as \emph{trajectory-level} RL objectives, because the reward is defined only at the terminal state (not decomposed across denoising steps) and the optimization is carried out over the full trajectory $\mathbf{x}_{0:T}$.
Typically, the RL fine-tuning objective is implemented as a KL-regularized Markov Decision Process (MDP), which we will discuss in detail in Section \ref{sec:single}.

 \subsection{Multi-objective reward model}
 In standard RL fine-tuning, the model is typically optimized with respect to one single reward function. As a result, the performance of this aligned model may degrade significantly when evaluated by another reward function. To address this problem, we consider the following multi-objective setting \cite{yang2019generalized,zhou2022anchor,shi2024decoding}, which assumes that there exists a set of reward functions $\{r_i\}_{i=1}^M$, representing $M$ distinct objectives. A normalized vector $w \in \Delta^{M-1}$ is used to represent the human preference among different objectives, where $\Delta^{M-1}$ denotes the probability simplex. For a user with preference vector $w$, the goal is to maximize the RL fine-tuning objective under the weighted reward function $r^w=\sum_{i=1}^Mw_ir_i$. Although fine-tuning with this weighted reward can yield strong performance, it is impractical to retrain a separate model for each possible preference $w$.
 
To overcome this limitation, denoising-time alignment is studied, where the core idea is to maximize the weighted reward exclusively through denoising by combining the reverse conditional distributions of a set of existing single-objective aligned diffusion models. For any preference vector $w$, denoising-time alignment requires no additional training: it directly runs the diffusion reverse process using the combined reverse conditionals.

\section{Main Results}\label{sec:method}
 \subsection{Step-level RL fine-tuning} \label{sec:single}
 As noted in Section~\ref{sec:rlhf}, $p_\theta(\mathbf{x}_{0})$ is usually intractable to evaluate for diffusion models, and therefore in Eq.~\eqref{eq:rlhf2},  the intractable marginal is replaced with step-level KL terms to keep each aligned reverse conditional near the pretrained model.
However, optimizing either Eq.~\eqref{eq:rlhf1} or Eq.~\eqref{eq:rlhf2} requires drawing data from the updated policy $\pi_\theta$; consequently, the optimal policy depends on advantages computed under that same (unknown) policy, which makes the target policy hard to track. Therefore, existing denoising-time methods based on Eq.~\eqref{eq:rlhf1} and Eq.~\eqref{eq:rlhf2} rely on approximations whose errors are difficult to quantify \cite{anonymous2025deradiff,cheng2025diffusion}. Trust Region Policy Optimization (TRPO)~\cite{schulman2015trust} resolves this circular dependence between the policy and advantage estimates by optimizing a surrogate objective that can be evaluated using data from a fixed reference policy. This surrogate admits provable trust-region improvement and yields an optimal update that depends only on advantages under the reference policy, which can be estimated and tracked reliably.
 Inspired by TRPO~\cite{schulman2015trust,zeng2024token}, we therefore reformulate the problem as a step-level RL objective, applying the decomposition to both the reward (via step-wise advantages) and the KL terms.

Recall the $T$-horizon MDP formulation for single-objective diffusion models~\cite{black2023training,fan2023dpok}. Let $s_t=\mathbf{x}_{T-t}$ denote the state and $a_t=\mathbf{x}_{T-t-1}$ the action. The initial state distribution is $P_0(s_0)=p_\theta(\mathbf{x}_{T})=\mathcal N(\mathbf{0,I})$. The dynamics are deterministic: $P(s_{t+1}|s_t,a_t)=\delta_{a_t}$, where $\delta_{z}$ denotes the Dirac distribution at $z$.  The reward function is non-zero only at the final step:
\begin{flalign}
    R(s_t, a_t) =
\begin{cases}
r(s_T), & t = T-1,\\
0, & \text{otherwise}.
\end{cases}\nonumber
\end{flalign}
Finally, the policy can be expressed as $\pi_\theta(a_t|s_t)=p_\theta(\mathbf{x}_{T-t-1}|\mathbf{x}_{T-t})$. 
Define the following value functions:
\begin{flalign}
    Q_\pi(s_t, a_t)
  &= \mathbb{E}_{\pi}\bigl[ r(s_T) \mid s_t, a_t \bigr],\nonumber\\
  V_\pi(s_t)
  &= \mathbb{E}_{ \pi}
    \bigl[ Q_\pi(s_t, a_t) \mid s_t \bigr],\nonumber\\
    A_\pi(s_t, a_t) &= Q_\pi(s_t, a_t) - V_\pi(s_t),\nonumber
\end{flalign}
where $Q_\pi$ is the state--action value function under policy $\pi$, and $V_\pi$ and $A_\pi$ are the corresponding value and advantage functions.

We then introduce our step-level RL objective. Motivated by the TRPO approach \citep{schulman2015trust}, we apply the decomposition to both the reward (via step-wise advantages) and the KL terms:
\begin{flalign}\label{eq:rlhf3}
  \max_{\pi_\theta} \quad&\mathbb{E}_{s_t \sim \mathcal D, z \sim \pi_\theta(\cdot \mid s_t)}
      \bigl[ A_{\pi_{\mathrm{pre}}}(s_t, z)\nonumber\\ 
    &- \lambda\mathbb{D}_{\mathrm{KL}}\bigl(
        \pi_\theta(\cdot \mid s_t)
        \,\|\, 
        \pi_{\mathrm{pre}}(\cdot \mid s_t)
      \bigr)\bigr],
\end{flalign}
where $\mathcal{D}$ is the dataset and $\pi_{\mathrm{pre}}$ is the policy induced by the pre-trained model. Compared with Eq.~\eqref{eq:rlhf2}, Eq.~\eqref{eq:rlhf3} replaces the terminal-state reward with the advantage function under the pre-trained policy $\pi_{\mathrm{pre}}$. Furthermore, Lemma~4.1 in \cite{zeng2024token} shows that maximizing Eq.~\eqref{eq:rlhf3} leads to policy improvement in expectation.

\cite{wallace2024diffusion} studies an RL fine-tuning variant that replaces the marginal KL regularizer $\mathbb D_{\mathrm{KL}}[p_\theta(\mathbf{x}_{0}) ||p_{\mathrm{pre}}(\mathbf{x}_{0})]$ in Eq.~\eqref{eq:rlhf1} with the trajectory-level KL regularizer $\mathbb D_{\mathrm{KL}}[p_\theta(\mathbf{x}_{0:T}) ||p_{\mathrm{pre}}(\mathbf{x}_{0:T})]$. 
Based on this formulation, a DPO-style method for diffusion models has been proposed, which directly aligns the model to human preference data. For our step-level RL formulation, in the following lemma, we also provide a corresponding DPO objective.
Let $\alpha_t := 1-\beta_t$ and $\bar{\alpha}_t := \prod_{i=1}^t \alpha_i$.  
Let $\mathbf{x}_0^w$ and $\mathbf{x}_0^l$ denote the winning and losing samples, respectively.  
For $*\in\{w,l\}$, define the forward diffusion process
\[
\mathbf{x}_t^{*}
= \sqrt{\bar{\alpha}_t}\,\mathbf{x}_0^{*}
+ \sqrt{1-\bar{\alpha}_t}\,\epsilon^{*},
\qquad
\epsilon^{*} \sim \mathcal{N}(\mathbf{0},\mathbf{I}),
\]
so that $\mathbf{x}_t^{*} \sim q(\mathbf{x}_t^{*}\mid \mathbf{x}_0^{*})$.
The corresponding signal-to-noise ratio is
\[
\lambda_t := \frac{\bar{\alpha}_t}{1-\bar{\alpha}_t}.
\]
\begin{lemma}\label{lemma:DPO}
The step-level DPO objective derived from Eq.~\eqref{eq:rlhf3} can be written as
\begin{flalign}
-\, &\mathbb{E}_{(\mathbf{x}_0^w,\mathbf{x}_0^l)\sim \mathcal{D},\,
t\sim \mathcal{U}(1,T),\,
\mathbf{x}_t^w\sim q(\mathbf{x}_t^w \mid \mathbf{x}_0^w),\,
\mathbf{x}_t^l\sim q(\mathbf{x}_t^l \mid \mathbf{x}_0^l)} \nonumber\\
&\Bigg[
\log \sigma\Bigg(
-\lambda T \,\omega(\lambda_t)\,
\Big(
\Delta_{\theta}^{w}
- \Delta_{\theta}^{l}
- \Delta_{\mathrm{diff}}
\Big)
\Bigg)
\Bigg],
\end{flalign}
where $\sigma(z)=\frac{1}{1+\exp(-z)}$ is the sigmoid function, $\mathcal U(1,T)$ denotes a  uniform distribution over $\{1,\cdots,T\}$ and
\begin{align}
\Delta_{\theta}^{*}
:&= \|\epsilon^{*}-\epsilon_\theta(\mathbf{x}_t^{*},t)\|_2^2
   - \|\epsilon^{*}-\epsilon_{\mathrm{pre}}(\mathbf{x}_t^{*},t)\|_2^2, \nonumber\\
\Delta_{\mathrm{diff}}
:&= \|\epsilon_\theta(\mathbf{x}_t^w,t)-\epsilon_{\mathrm{pre}}(\mathbf{x}_t^w,t)\|_2^2\nonumber\\
   &- \|\epsilon_\theta(\mathbf{x}_t^l,t)-\epsilon_{\mathrm{pre}}(\mathbf{x}_t^l,t)\|_2^2, \nonumber
\end{align}
$\omega(\lambda_t)$ denotes a weighting function (often chosen to be constant in practice~\cite{ho2020denoising,kingma2021variational}), and
$\epsilon_*(\mathbf{x}_t,t)$ is a function approximator predicting the noise term such that the corresponding reverse-process mean is given by
\[
\mu_*(\mathbf{x}_t,t)
= \frac{1}{\sqrt{\alpha_t}}
\left(
\mathbf{x}_t
- \frac{\beta_t}{\sqrt{1-\bar{\alpha}_t}}\,
\epsilon_*(\mathbf{x}_t,t)
\right).
\]
\end{lemma}

The full proof is shown in Appendix \ref{proof:dpo}.

\begin{remark}
    Unlike the diffusion DPO objective that is derived from a variant of Eq.~\eqref{eq:rlhf1} with a joint KL
regularizer~\cite{wallace2024diffusion}, our step-level DPO loss contains an extra step-wise regularization term $\Delta_{\mathrm{diff}}$ inside the $\log \sigma(\cdot)$. This term explicitly encourages the aligned model to remain closer to the pretrained model on less-preferred trajectories than on preferred ones at each denoising step. Intuitively, when the training signal is positive (preferred trajectory), the model is allowed to move away from the pretrained model, whereas when the signal is negative (less-preferred trajectory), the penalty keeps the aligned model close to the pretrained weights, making it safe to leave those behaviors unchanged.
\end{remark}

\subsection{Multi-objective alignment}
In Section~\ref{sec:single}, we developed a step-level RL fine-tuning framework for a single reward function. We now extend this framework to the multi-objective setting and design a retraining-free method that efficiently aligns diffusion models with diverse human preferences, without introducing any approximation error.

In the multi-objective setting, we assume a collection of reward functions
$\{r_i\}_{i=1}^M$ corresponding to $M$ distinct objectives. For each objective
$i$, we assume access to the policy $\pi_i$ and the corresponding reverse distribution $p_i$ that maximizes the single-objective
step-level RL objective in Eq.~\eqref{eq:rlhf3} with reward $r_i$. We then show
that, for any weighted reward
$r^w = \sum_{i=1}^M w_i r_i$ with preference weights $w \in \Delta^{M-1}$, we
can construct the optimal reverse distribution $p_w$ purely by combining the distributions
$\{p_i\}_{i=1}^M$, without any additional training. The corresponding denoising-time sampling procedure is summarized in
Algorithm~\ref{alg:multi_obj_denoise}. 
\begin{algorithm}[t]
\caption{Multi-objective Step-level Denoising-time Diffusion Alignment (MSDDA)}
\label{alg:multi_obj_denoise}
\begin{algorithmic}[1]
\REQUIRE Aligned models $\{p_i(\mathbf{x}_{t-1}\mid \mathbf{x}_t)\}_{i=1}^M$,
         preference weights $w \in \Delta^{M-1}$,
         inference steps $T$
\STATE Sample $\mathbf{x}_{T} \sim \mathcal{N}(\mathbf{0},\mathbf{I})$ %\COMMENT{initial latent}
\FOR{$k = T, \dots, 1$}
    \STATE $t \gets t_k,\;\; t_{\text{prev}} \gets t_{k-1}$
    \FOR{$i = 1$ \textbf{to} $M$}
       \STATE $(\mu_i, \sigma_i^2) \gets \text{parameters of }
           p_i(\mathbf{x}_{t-1}\mid \mathbf{x}_t)
           = \mathcal N(\mu_i,\sigma_i^2 \mathbf{I})$
        %\COMMENT{$p_{i}(\mathbf{x}_{t-1}\mid \mathbf{x}_{t})$ is $\mathcal N(\mu_i,\sigma_i^2 I)$}
    \ENDFOR
    \STATE $\sigma_{\text{new}}^{-2}
            \gets \displaystyle\sum_{i=1}^M \frac{w_i}{\sigma_i^2}$
    \STATE $\sigma_{\text{new}}^2 \gets 1 / \sigma_{\text{new}}^{-2}$
    \STATE $\mu_{\text{new}}
            \gets \sigma_{\text{new}}^2
                \displaystyle\sum_{i=1}^M \frac{w_i}{\sigma_i^2}\,\mu_i$
    \STATE Sample $z \sim \mathcal{N}(\mathbf{0},\mathbf{I})$
    \STATE $\mathbf{x}_{t_{\text{prev}}} \gets \mu_{\text{new}} + \sigma_{\text{new}}\, z$
\ENDFOR
\STATE\textbf{return} $\mathbf{x}_{0}$
\end{algorithmic}
\end{algorithm}

Before presenting our main result, we state a lemma that characterizes the optimal policy for step-level RL objectives.

\begin{lemma}[Lemma~4.2, \cite{zeng2024token}]\label{lemma1}
For any reward function $r$, the optimal policy for the step-level RL objective
in Eq.~\eqref{eq:rlhf3} has the closed-form expression
\begin{flalign}
\pi_\theta^\star(z \mid s_t)
  = \frac{
      \pi_{\mathrm{pre}}(z \mid s_t)
      \exp\bigl(Q_{\pi_{\mathrm{pre}}}(s_t, z)/\lambda\bigr)
    }{
      Z(s_t)
    },
\end{flalign}
where
$Z(s_t)
  = \int
      \pi_{\mathrm{pre}}(z \mid s_t)
      \exp\bigl(Q_{\pi_{\mathrm{pre}}}(s_t, z)/\lambda\bigr)\nonumber
    \, dz$
is the normalizing constant.
\end{lemma}

This lemma is from Lemma~4.2 of \cite{zeng2024token}. For completeness, we include the proof in Appendix \ref{proof:lemma}.

A key advantage of our step-level RL formulation is that, once we have obtained the aligned models for each individual objective, we can derive closed-form solutions for the optimal policy corresponding to any weighted reward
$r^w = \sum_{i=1}^M w_i r_i$.
\begin{theorem}\label{theorem1}
Let $\pi_i$ be the optimal policy of Eq.~\eqref{eq:rlhf3} with reward $r_i$, and $p_i$ is the corresponding reverse distribution.
At each step $t$, let $(\mu_i, \sigma_i^2)$ denote the mean and variance of
$p_{i}(\mathbf{x}_{t-1}\mid \mathbf{x}_{t})$ so that
\[
p_{i}(\mathbf{x}_{t-1}\mid \mathbf{x}_{t})
  = \mathcal N(\mu_i,\sigma_i^2 \mathbf{I}).
\]
Then, for any $w\in \Delta^{M-1}$ and
$r^w=\sum_{i=1}^M w_i r_i$, the step-level optimal posterior is
\begin{flalign}
{p}_w(\mathbf{x}_{t-1}\mid \mathbf{x}_{t})
  =
  \frac{
    \prod_{i=1}^M {p}_i^{w_i}(\mathbf{x}_{t-1}\mid \mathbf{x}_{t})
  }{
    \displaystyle
    \int
    \prod_{i=1}^M {p}_i^{w_i}(\mathbf{x}'_{t-1}\mid \mathbf{x}_{t}) \, d\mathbf{x}'_{t-1}
  },
\label{eq:theorem1_post}
\end{flalign}
which is Gaussian with closed-form parameters
\begin{flalign}\label{eq:t1}
    \sigma_w^2 &= \left(\sum_{i=1}^M \frac{w_i}{\sigma_i^2}\right)^{-1}, \nonumber\\
    \mu_{w} &= \sigma_w^2\sum_{i=1}^M \frac{w_i}{\sigma_i^2}\,\mu_i.
\end{flalign}
\end{theorem}
We provide a proof sketch below; the full proof is deferred to Appendix \ref{proof:the}.
\begin{proof}[Proof sketch]
By Lemma~\ref{lemma1}, the optimal policy $\pi_i$ for reward $r_i$ satisfies
\begin{flalign}\label{eq:p1}
\pi_i \propto \exp\bigl(Q^i_{\pi_{\mathrm{pre}}}(s_t, z)/\lambda\bigr),
\end{flalign}
where $Q^i_{\pi_{\mathrm{pre}}}$ is the state--action value function under the
reference policy and reward $r_i$. Since the combined reward is
$r^w = \sum_{i=1}^M w_i r_i$, we have
\begin{flalign}\label{eq:p2}
Q^w_{\pi_{\mathrm{pre}}}(s_t, z)
  = \sum_{i=1}^M w_i Q^i_{\pi_{\mathrm{pre}}}(s_t, z).
\end{flalign}
Applying Lemma~\ref{lemma1} with reward $r^w$ and using
Eqs.~\eqref{eq:p1}, ~\eqref{eq:p2} and the definition that $\pi_i(a_t|s_t)=p_i(\mathbf{x}_{T-t-1}|\mathbf{x}_{T-t})$ for $s_t=\mathbf{x}_{T-t}$ and $a_t=\mathbf{x}_{T-t-1}$, we obtain that
\begin{flalign}
p_w
&\propto \exp\bigl(Q^w_{\pi_{\mathrm{pre}}}(s_t, z)/\lambda\bigr) \nonumber\\
&\propto \prod_{i=1}^M
   \exp\bigl(Q^i_{\pi_{\mathrm{pre}}}(s_t, z)/\lambda\bigr)^{w_i} \nonumber\\
&\propto \prod_{i=1}^M p_i^{w_i}.
\end{flalign}
Since each $p_i$ is Gaussian, their weighted product is also Gaussian, and we can calculate the mean and variance, which are in Eq.~\eqref{eq:t1}.
\end{proof}

Theorem~\ref{theorem1} provides closed-form expressions for the mean and variance
of the optimal policy for every preference vector $w \in \Delta^{M-1}$. Thus,
our denoising-time alignment procedure requires no retraining: we only need to
fuse the single-objective posteriors $\{p_i\}_{i=1}^M$.

\begin{figure*}[!htb]
  \centering
  \caption{Quantitative comparison of our proposed MSDDA method and baseline methods. The results for CoDe and RGG are obtained from \cite{cheng2025diffusion}.}
  \label{tab:db-mpa-comparison}
  \vspace{1mm}
  \small
  % added an extra 'l' column after the row-label so we have a dedicated "w=" column
  \begin{tabular}{@{} l l  % <--- first = label, second = w= column
                   cc    % SD: r1 r2 (cols 3-4)
                   cc    % DB-MPA: r1 r2 (cols 5-6)
                   cc    % RS: r1 r2
                   cc    % CoDe: r1 r2
                   cc    % RGG: r1 r2
                   @{} }
    \toprule
    % top header: leave first two heading cells empty (label + w= column)
    & & \multicolumn{2}{c}{SD} &
          \multicolumn{2}{c}{MSDDA} &
          \multicolumn{2}{c}{RS} &
          \multicolumn{2}{c}{CoDe} &
          \multicolumn{2}{c}{RGG} \\
    \cmidrule(lr){3-4}\cmidrule(lr){5-6}\cmidrule(lr){7-8}\cmidrule(lr){9-10}\cmidrule(lr){11-12}
    % second header row: label the two small leading columns, then r1 r2 repeated
    & & $r_1$ & $r_2$  & $r_1$ & $r_2$ & $r_1$ & $r_2$ & $r_1$ & $r_2$ & $r_1$ & $r_2$ \\
    \midrule
    % SD row (no per-w info) -> we fill the w= column with empty cell
    %SD & & 0.22 & -0.15 & \multicolumn{2}{c}{} & \multicolumn{2}{c}{} & \multicolumn{2}{c}{} & \multicolumn{2}{c}{} \\
    \addlinespace[2pt]
    % Reward rows: first column is the multirow label, second column holds the w= value
    \multirow{3}{*}{Reward (↑)}
      & $w=0.2$ & 0.22 & -0.15 & \textbf{0.40} & \textbf{0.52} & 0.24 & 0.37 & \textbf{0.40} & 0.05 & 0.21 & 0.42 \\
      & $w=0.5$  & 0.22 & -0.15 & \textbf{0.61} & \textbf{0.26} & 0.38 & 0.04 & 0.60 & 0.01 & 0.23 & 0.23 \\
      & $w=0.8$  & 0.22 & -0.15 & 0.65 & \textbf{0.02} & 0.65 & -0.12 & \textbf{0.66} & -0.07 & 0.27 & {0.00} \\
    \midrule
    Inference Time (↓ sec/img)
      & & \textbf{5.08}  & & \underline{10.14} & & \textbf{5.08} & & 185.26 & & 121.58 & \\
    \bottomrule
  \end{tabular}
\end{figure*}

Finally, we highlight the novelty of our method compared with existing denoising-time alignment
methods such as DB-MPA~\cite{cheng2025diffusion} and DERADIFF~\cite{anonymous2025deradiff}. These methods are 
derived from Eq.~\eqref{eq:rlhf1}, whose optimal policy is hard to track due to the circular
dependence between the policy and advantage estimates. To obtain a
denoising-time objective, \cite{cheng2025diffusion,anonymous2025deradiff}
discard several terms in the diffusion dynamics, resulting in approximation
errors that are hard to quantify. In contrast, our step-level formulation resolves this intractability problem,
which allows us to derive an exact closed-form denoising objective without any
approximations. In Appendix \ref{approxi}, we provide detailed discussions about how the approximated errors are introduced in  DB-MPA and DERADIFF. 
From an algorithmic perspective, our sampler coincides with the
procedure used in DERADIFF~\cite{anonymous2025deradiff}, but our analysis provides explicit theoretical
guarantees. Moreover, DB-MPA~\cite{cheng2025diffusion} can be viewed as a
special case of our method that \textbf{requires all objectives to share the same variance
$\sigma_i^2$}. Because of this assumption, DB-MPA is not applicable when objectives have different variances. By allowing arbitrary $\sigma_i^2$, our framework strictly generalizes DB-MPA and remains valid in the heterogeneous-variance setting.

\begin{figure}[!tb]
    \centering
    \includegraphics[width=0.9\linewidth]{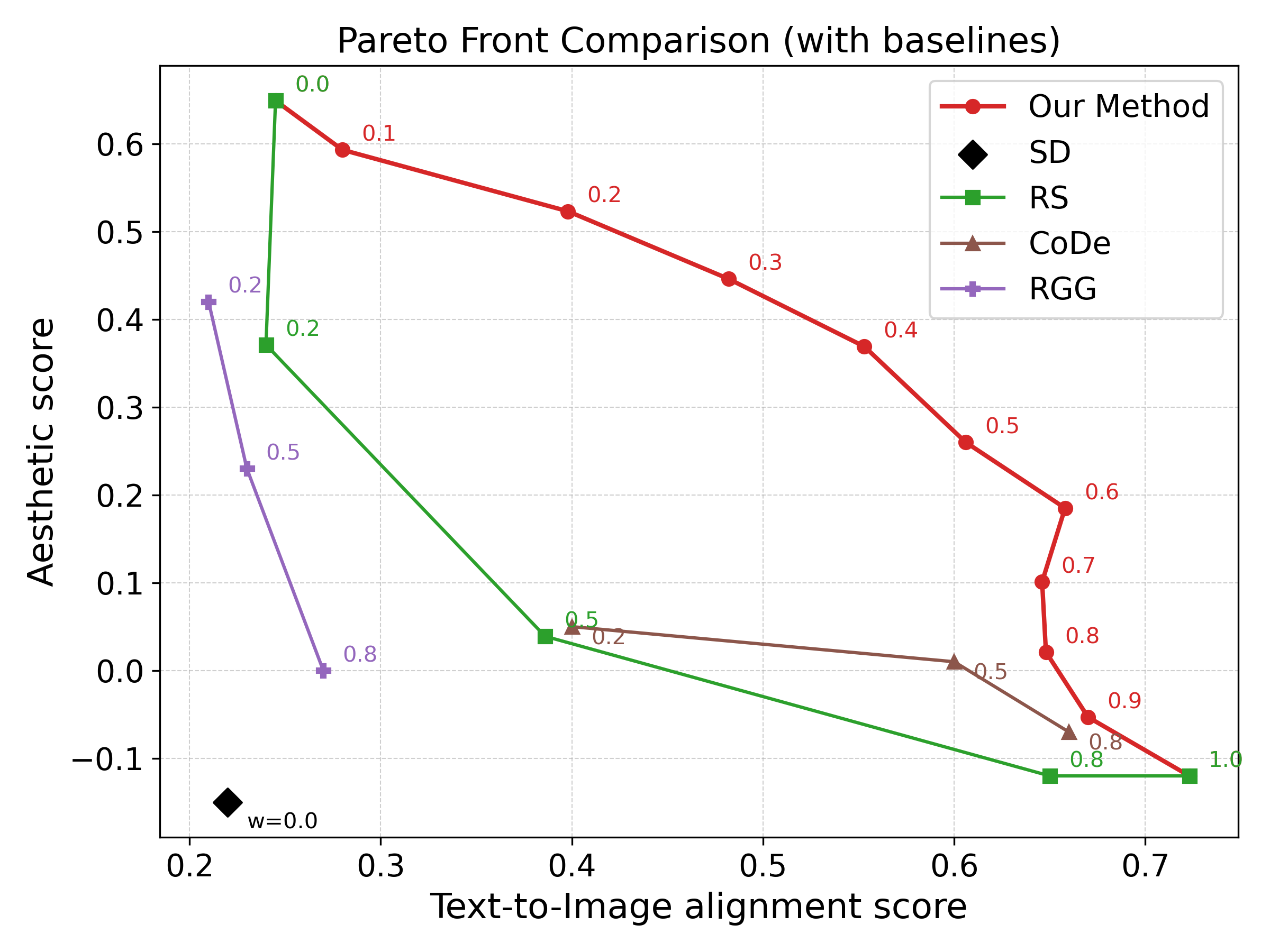}
    \caption{Pareto front of our proposed method and the baselines with different preference weight $w$. The results for CoDe and RGG are obtained from \cite{cheng2025diffusion}.}
    \label{fig:placeholder2}
\end{figure}

\section{Experiments}
We then present numerical results that validate the effectiveness of the proposed MSDDA algorithm, where $\lambda=0.1$ and other experiment details can be found in Appendix \ref{sec:detail}.

\textbf{Prompt dataset:} We adopt the same prompt dataset as in \cite{cheng2025diffusion}, which is a subset of DrawBench \cite{saharia2022photorealistic} restricted to the ``color'' category. To obtain test prompts that do not appear in the training data, GPT-4 is used to synthesize novel color–object and object–object combinations derived from the training set. 

\textbf{Diffusion and reward models:} Our experiments use Stable Diffusion v1.5 \cite{rombach2022high} as the pre-trained generative model. Generated images are evaluated using two reward models: (1) ImageReward \cite{xu2023imagereward} for measuring text–image alignment (higher is better), and (2) VILA \cite{ke2023vila} for assessing aesthetic quality. Following \cite{cheng2025diffusion}, we rescale VILA scores with the linear mapping $r \mapsto 4r-2$ so they share a similar range as ImageReward. For each test prompt, we sample $32$ images using seeds $0-31$, compute each reward for every sample, and report the average score.

\textbf{Baselines:} We compare MSDDA against several baselines: (1) Stable Diffusion (SD) v1.5 \cite{rombach2022high}, which serves as a pretrained baseline; (2) Reward Gradient-based Guidance (RGG) \cite{kim2025test}, which leverages reward values to guide denoising; (3) CoDe \cite{singh2025code}, a method that uses the gradient of the reward to steer sampling; and (4) Reward Soup (RS) \cite{rame2023rewarded}. In our reward-free setup, we have two aligned Stable Diffusion models, each optimized with a different reward and with distinct denoising variance schedules; because RS combines model parameters linearly and the two models share the same architecture, it can be applied in this setting. Another reward-free method, DB-MPA \cite{cheng2025diffusion}, requires identical denoising time variances across models and, therefore, is incompatible with our experimental configuration.

\begin{figure*}[!t]
  \centering
\includegraphics[width=0.7\textwidth]{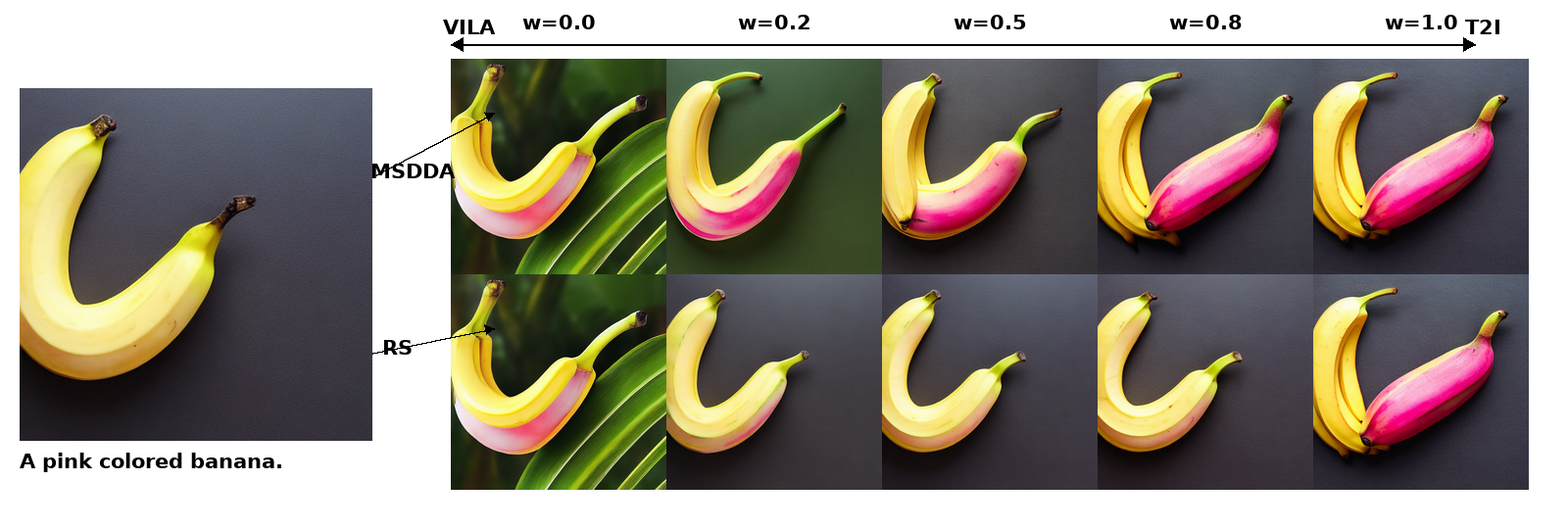}
\includegraphics[width=0.7\textwidth]{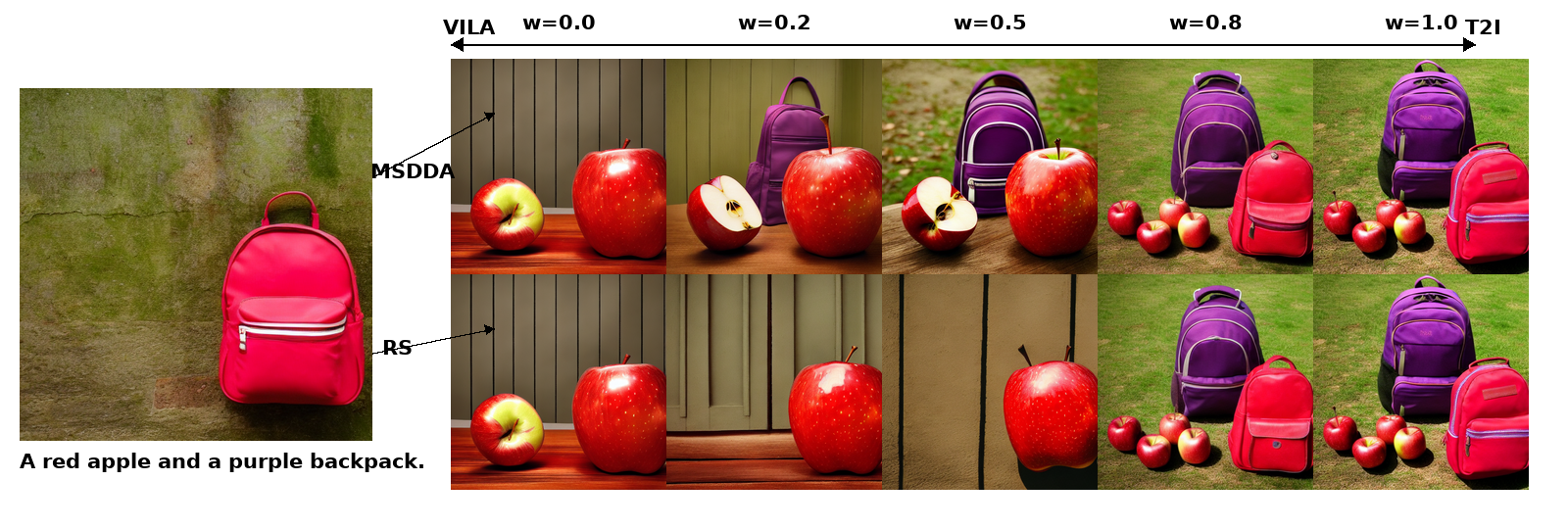}
\includegraphics[width=0.7\textwidth]{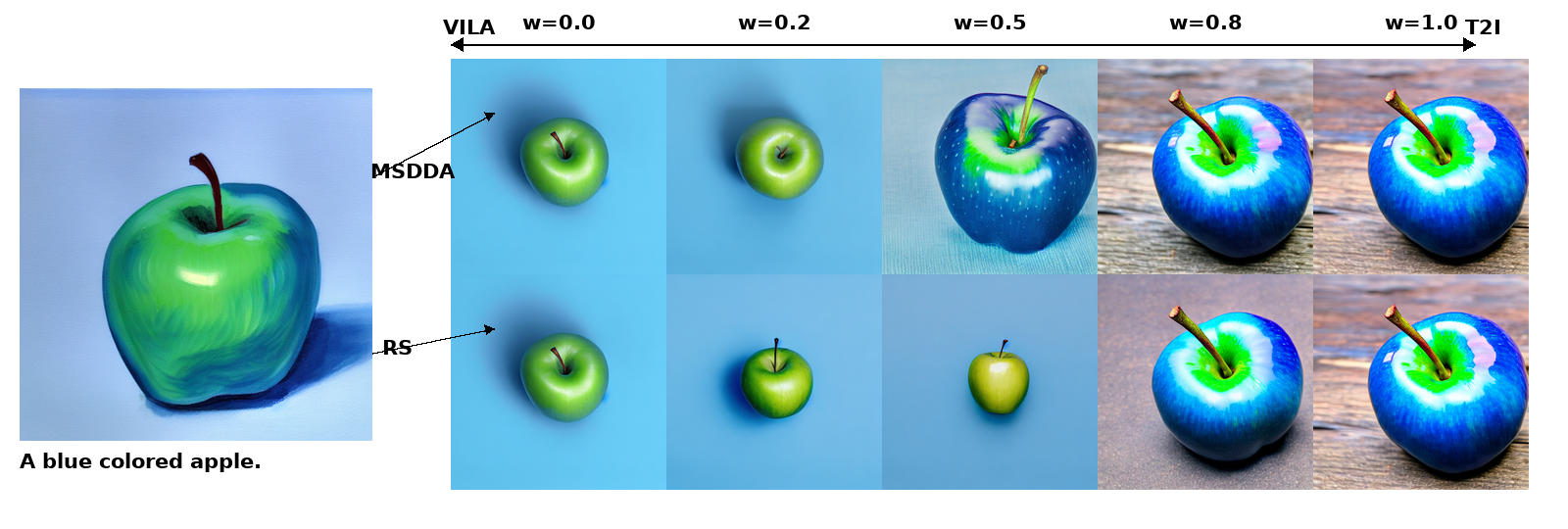}
  \caption{Image samples generated by our proposed method and RS under different preference weights $w$
The left image is generated by Stable Diffusion using the prompt shown in the lower-left corner. The top row shows images generated by our proposed MSDDA, while the bottom row shows images generated by RS.}\label{fig:ex}
\end{figure*}

In Fig. \ref{fig:placeholder2} we plot the Pareto front for our method and the baselines across different preference weights $w$. The horizontal axis shows the ImageReward score, and the vertical axis shows the VILA score. For a given weight $w$, the scalarized target reward is $r^w=wr_1+(1-w)r_2$, where $r_1$ denotes ImageReward and $r_2$ denotes VILA. Table \ref{tab:db-mpa-comparison} lists the numeric reward values for each $w$ and also reports the generation time required per image. From the numerical results, we can observe that
\begin{itemize}
    \item Performance. MSDDA consistently outperforms the baselines. For most preference weights $w$, our method achieves the highest scores on both reward models. The pretrained SD model struggles on several test prompts (e.g., “A purple colored dog” or “A blue colored apple”), which are uncommon in real-world data; all baselines improve over the SD baseline, demonstrating their ability to enhance denoising and alignment.
    \item Latency. SD and RS are the fastest methods because they invoke only a single diffusion model. Our method has the second-shortest runtime, roughly twice the denoising time of SD/RS—consistent with MSDDA’s requirement to run two diffusion models. RGG and Code incur larger runtimes because they evaluate reward functions during denoising.
\end{itemize}

In Fig. \ref{fig:ex}, we show representative images produced by our MSDDA and by RS under different preference weights $w$ to directly show the model performance. The base SD model often fails to follow the prompt (for example, producing a pink banana, a purple backpack with a red apple, or a blue apple). Both MSDDA and RS use the same aligned diffusion backbones, but MSDDA achieves noticeably better alignment with the prompts. For instance, at
$w\in [0.2, 0.5, 0.8]$, MSDDA reliably generates the pink banana and the purple backpack with a red apple, while RS often does not. Additional examples appear in Appendix \ref{sec:detail}.

\section{Conclusion}
This paper addresses the multi-objective denoising-time alignment problem. To solve the intractability of the optimal policy in existing RL fine-tuning formulations, we model the task as a step-level RL problem and derive a corresponding DPO objective that encourages the aligned model to remain closer to the pretrained model on less-preferred trajectories than on preferred ones. 
Building on this formulation, we propose MSDDA, a retraining-free fusion method that leverages closed-form step-level solutions to combine existing aligned models and provably maximizes the multi-objective reward without approximation error or access to individual reward functions. Extensive numerical results are conducted to validate the effectiveness of our proposed MSDDA.

% different assumptions. 
% Lastly, the proposed algorithms are expected to be applicable to various applications, including reinforcement learning and autonomous driving.

\newpage
\bibliography{arxiv2}

@InProceedings{pmlr-v37-sohl-dickstein15,
  title = 	 {Deep Unsupervised Learning using Nonequilibrium Thermodynamics},
  author = 	 {Sohl-Dickstein, Jascha and Weiss, Eric and Maheswaranathan, Niru and Ganguli, Surya},
  booktitle = 	 {Proceedings of the 32nd International Conference on Machine Learning},
  pages = 	 {2256--2265},
  year = 	 {2015},
  editor = 	 {Bach, Francis and Blei, David},
  volume = 	 {37},
  series = 	 {Proceedings of Machine Learning Research},
  address = 	 {Lille, France},
  month = 	 {07--09 Jul},
  publisher =    {PMLR},
  url = 	 {https://proceedings.mlr.press/v37/sohl-dickstein15.html}
}

@article{ho2020denoising,
  title={Denoising diffusion probabilistic models},
  author={Ho, Jonathan and Jain, Ajay and Abbeel, Pieter},
  journal={Advances in neural information processing systems},
  volume={33},
  pages={6840--6851},
  year={2020}
}

@article{fan2023dpok,
  title={Dpok: Reinforcement learning for fine-tuning text-to-image diffusion models},
  author={Fan, Ying and Watkins, Olivia and Du, Yuqing and Liu, Hao and Ryu, Moonkyung and Boutilier, Craig and Abbeel, Pieter and Ghavamzadeh, Mohammad and Lee, Kangwook and Lee, Kimin},
  journal={Advances in Neural Information Processing Systems},
  volume={36},
  pages={79858--79885},
  year={2023}
}

@article{ouyang2022training,
  title={Training language models to follow instructions with human feedback},
  author={Ouyang, Long and Wu, Jeffrey and Jiang, Xu and Almeida, Diogo and Wainwright, Carroll and Mishkin, Pamela and Zhang, Chong and Agarwal, Sandhini and Slama, Katarina and Ray, Alex and others},
  journal={Advances in neural information processing systems},
  volume={35},
  pages={27730--27744},
  year={2022}
}

@article{uehara2024understanding,
  title={Understanding reinforcement learning-based fine-tuning of diffusion models: A tutorial and review},
  author={Uehara, Masatoshi and Zhao, Yulai and Biancalani, Tommaso and Levine, Sergey},
  journal={arXiv preprint arXiv:2407.13734},
  year={2024}
}

@article{yang2019generalized,
  title={A generalized algorithm for multi-objective reinforcement learning and policy adaptation},
  author={Yang, Runzhe and Sun, Xingyuan and Narasimhan, Karthik},
  journal={Advances in neural information processing systems},
  volume={32},
  year={2019}
}

@article{zhou2022anchor,
  title={Anchor-changing regularized natural policy gradient for multi-objective reinforcement learning},
  author={Zhou, Ruida and Liu, Tao and Kalathil, Dileep and Kumar, PR and Tian, Chao},
  journal={Advances in neural information processing systems},
  volume={35},
  pages={13584--13596},
  year={2022}
}

@article{shi2024decoding,
  title={Decoding-time language model alignment with multiple objectives},
  author={Shi, Ruizhe and Chen, Yifang and Hu, Yushi and Liu, Alisa and Hajishirzi, Hanna and Smith, Noah A and Du, Simon S},
  journal={Advances in Neural Information Processing Systems},
  volume={37},
  pages={48875--48920},
  year={2024}
}

@article{rafailov2023direct,
  title={Direct preference optimization: Your language model is secretly a reward model},
  author={Rafailov, Rafael and Sharma, Archit and Mitchell, Eric and Manning, Christopher D and Ermon, Stefano and Finn, Chelsea},
  journal={Advances in neural information processing systems},
  volume={36},
  pages={53728--53741},
  year={2023}
}

@article{liu2024decoding,
  title={Decoding-time realignment of language models},
  author={Liu, Tianlin and Guo, Shangmin and Bianco, Leonardo and Calandriello, Daniele and Berthet, Quentin and Llinares, Felipe and Hoffmann, Jessica and Dixon, Lucas and Valko, Michal and Blondel, Mathieu},
  journal={arXiv preprint arXiv:2402.02992},
  year={2024}
}

@inproceedings{schulman2015trust,
  title={Trust region policy optimization},
  author={Schulman, John and Levine, Sergey and Abbeel, Pieter and Jordan, Michael and Moritz, Philipp},
  booktitle={International conference on machine learning},
  pages={1889--1897},
  year={2015},
  organization={PMLR}
}

@article{zeng2024token,
  title={Token-level direct preference optimization},
  author={Zeng, Yongcheng and Liu, Guoqing and Ma, Weiyu and Yang, Ning and Zhang, Haifeng and Wang, Jun},
  journal={arXiv preprint arXiv:2404.11999},
  year={2024}
}

@article{black2023training,
  title={Training diffusion models with reinforcement learning},
  author={Black, Kevin and Janner, Michael and Du, Yilun and Kostrikov, Ilya and Levine, Sergey},
  journal={arXiv preprint arXiv:2305.13301},
  year={2023}
}

@inproceedings{wallace2024diffusion,
  title={Diffusion model alignment using direct preference optimization},
  author={Wallace, Bram and Dang, Meihua and Rafailov, Rafael and Zhou, Linqi and Lou, Aaron and Purushwalkam, Senthil and Ermon, Stefano and Xiong, Caiming and Joty, Shafiq and Naik, Nikhil},
  booktitle={Proceedings of the IEEE/CVF Conference on Computer Vision and Pattern Recognition},
  pages={8228--8238},
  year={2024}
}

@article{kingma2021variational,
  title={Variational diffusion models},
  author={Kingma, Diederik and Salimans, Tim and Poole, Ben and Ho, Jonathan},
  journal={Advances in neural information processing systems},
  volume={34},
  pages={21696--21707},
  year={2021}
}

@article{cheng2025diffusion,
  title={Diffusion Blend: Inference-Time Multi-Preference Alignment for Diffusion Models},
  author={Cheng, Min and Doudi, Fatemeh and Kalathil, Dileep and Ghavamzadeh, Mohammad and Kumar, Panganamala R},
  journal={arXiv preprint arXiv:2505.18547},
  year={2025}
}

@article{saharia2022photorealistic,
  title={Photorealistic text-to-image diffusion models with deep language understanding},
  author={Saharia, Chitwan and Chan, William and Saxena, Saurabh and Li, Lala and Whang, Jay and Denton, Emily L and Ghasemipour, Kamyar and Gontijo Lopes, Raphael and Karagol Ayan, Burcu and Salimans, Tim and others},
  journal={Advances in neural information processing systems},
  volume={35},
  pages={36479--36494},
  year={2022}
}

@article{ramesh2022hierarchical,
  title={Hierarchical text-conditional image generation with clip latents},
  author={Ramesh, Aditya and Dhariwal, Prafulla and Nichol, Alex and Chu, Casey and Chen, Mark},
  journal={arXiv preprint arXiv:2204.06125},
  volume={1},
  number={2},
  pages={3},
  year={2022}
}

@inproceedings{rombach2022high,
  title={High-resolution image synthesis with latent diffusion models},
  author={Rombach, Robin and Blattmann, Andreas and Lorenz, Dominik and Esser, Patrick and Ommer, Bj{\"o}rn},
  booktitle={Proceedings of the IEEE/CVF conference on computer vision and pattern recognition},
  pages={10684--10695},
  year={2022}
}

@article{lee2023aligning,
  title={Aligning text-to-image models using human feedback},
  author={Lee, Kimin and Liu, Hao and Ryu, Moonkyung and Watkins, Olivia and Du, Yuqing and Boutilier, Craig and Abbeel, Pieter and Ghavamzadeh, Mohammad and Gu, Shixiang Shane},
  journal={arXiv preprint arXiv:2302.12192},
  year={2023}
}

@article{clark2023directly,
  title={Directly fine-tuning diffusion models on differentiable rewards},
  author={Clark, Kevin and Vicol, Paul and Swersky, Kevin and Fleet, David J},
  journal={arXiv preprint arXiv:2309.17400},
  year={2023}
}

@article{wang2025theoretical,
  title={Theoretical study of conflict-avoidant multi-objective reinforcement learning},
  author={Wang, Yudan and Xiao, Peiyao and Ban, Hao and Ji, Kaiyi and Zou, Shaofeng},
  journal={IEEE Transactions on Information Theory},
  year={2025},
  publisher={IEEE}
}

@article{rame2023rewarded,
  title={Rewarded soups: towards pareto-optimal alignment by interpolating weights fine-tuned on diverse rewards},
  author={Rame, Alexandre and Couairon, Guillaume and Dancette, Corentin and Gaya, Jean-Baptiste and Shukor, Mustafa and Soulier, Laure and Cord, Matthieu},
  journal={Advances in Neural Information Processing Systems},
  volume={36},
  pages={71095--71134},
  year={2023}
}

@inproceedings{han2023training,
  title={Training-free multi-objective diffusion model for 3d molecule generation},
  author={Han, Xu and Shan, Caihua and Shen, Yifei and Xu, Can and Yang, Han and Li, Xiang and Li, Dongsheng},
  booktitle={The Twelfth International Conference on Learning Representations},
  year={2023}
}

@article{kim2025test,
  title={Test-time alignment of diffusion models without reward over-optimization},
  author={Kim, Sunwoo and Kim, Minkyu and Park, Dongmin},
  journal={arXiv preprint arXiv:2501.05803},
  year={2025}
}

@article{ye2024tfg,
  title={Tfg: Unified training-free guidance for diffusion models},
  author={Ye, Haotian and Lin, Haowei and Han, Jiaqi and Xu, Minkai and Liu, Sheng and Liang, Yitao and Ma, Jianzhu and Zou, James Y and Ermon, Stefano},
  journal={Advances in Neural Information Processing Systems},
  volume={37},
  pages={22370--22417},
  year={2024}
}

@article{singh2025code,
  title={CoDe: Blockwise Control for Denoising Diffusion Models},
  author={Singh, Anuj and Mukherjee, Sayak and Beirami, Ahmad and Jamali-Rad, Hadi},
  journal={arXiv preprint arXiv:2502.00968},
  year={2025}
}

@inproceedings{wu2023human,
  title={Human preference score: Better aligning text-to-image models with human preference},
  author={Wu, Xiaoshi and Sun, Keqiang and Zhu, Feng and Zhao, Rui and Li, Hongsheng},
  booktitle={Proceedings of the IEEE/CVF International Conference on Computer Vision},
  pages={2096--2105},
  year={2023}
}

@article{desideri2012multiple,
  title={Multiple-gradient descent algorithm (MGDA) for multiobjective optimization},
  author={D{\'e}sid{\'e}ri, Jean-Antoine},
  journal={Comptes Rendus Mathematique},
  volume={350},
  number={5-6},
  pages={313--318},
  year={2012},
  publisher={Elsevier}
}

@article{xiao2023direction,
  title={Direction-oriented multi-objective learning: Simple and provable stochastic algorithms},
  author={Xiao, Peiyao and Ban, Hao and Ji, Kaiyi},
  journal={Advances in Neural Information Processing Systems},
  volume={36},
  pages={4509--4533},
  year={2023}
}

@article{chen2023three,
  title={Three-way trade-off in multi-objective learning: Optimization, generalization and conflict-avoidance},
  author={Chen, Lisha and Fernando, Heshan and Ying, Yiming and Chen, Tianyi},
  journal={Advances in Neural Information Processing Systems},
  volume={36},
  pages={70045--70093},
  year={2023}
}

@inproceedings{
zhang2025mgda,
title={{MGDA} Converges under Generalized Smoothness, Provably},
author={Qi Zhang and Peiyao Xiao and Shaofeng Zou and Kaiyi Ji},
booktitle={The Thirteenth International Conference on Learning Representations},
year={2025},
url={https://openreview.net/forum?id=wgDB1QuxIA}
}

@article{roijers2013survey,
  title={A survey of multi-objective sequential decision-making},
  author={Roijers, Diederik M and Vamplew, Peter and Whiteson, Shimon and Dazeley, Richard},
  journal={Journal of Artificial Intelligence Research},
  volume={48},
  pages={67--113},
  year={2013}
}

@article{hao2023optimizing,
  title={Optimizing prompts for text-to-image generation},
  author={Hao, Yaru and Chi, Zewen and Dong, Li and Wei, Furu},
  journal={Advances in Neural Information Processing Systems},
  volume={36},
  pages={66923--66939},
  year={2023}
}

@InProceedings{Yu_2023_ICCV,
    author    = {Yu, Jiwen and Wang, Yinhuai and Zhao, Chen and Ghanem, Bernard and Zhang, Jian},
    title     = {FreeDoM: Training-Free Energy-Guided Conditional Diffusion Model},
    booktitle = {Proceedings of the IEEE/CVF International Conference on Computer Vision (ICCV)},
    month     = {October},
    year      = {2023},
    pages     = {23174-23184}
}

@InProceedings{pmlr-v202-song23k,
  title = 	 {Loss-Guided Diffusion Models for Plug-and-Play Controllable Generation},
  author =       {Song, Jiaming and Zhang, Qinsheng and Yin, Hongxu and Mardani, Morteza and Liu, Ming-Yu and Kautz, Jan and Chen, Yongxin and Vahdat, Arash},
  booktitle = 	 {Proceedings of the 40th International Conference on Machine Learning},
  pages = 	 {32483--32498},
  year = 	 {2023},
  editor = 	 {Krause, Andreas and Brunskill, Emma and Cho, Kyunghyun and Engelhardt, Barbara and Sabato, Sivan and Scarlett, Jonathan},
  volume = 	 {202},
  series = 	 {Proceedings of Machine Learning Research},
  month = 	 {23--29 Jul},
  publisher =    {PMLR},
  url = 	 {https://proceedings.mlr.press/v202/song23k.html}
}

@inproceedings{bansal2023universal,
  title={Universal guidance for diffusion models},
  author={Bansal, Arpit and Chu, Hong-Min and Schwarzschild, Avi and Sengupta, Soumyadip and Goldblum, Micah and Geiping, Jonas and Goldstein, Tom},
  booktitle={Proceedings of the IEEE/CVF conference on computer vision and pattern recognition},
  pages={843--852},
  year={2023}
}

@article{li2024derivative,
  title={Derivative-free guidance in continuous and discrete diffusion models with soft value-based decoding},
  author={Li, Xiner and Zhao, Yulai and Wang, Chenyu and Scalia, Gabriele and Eraslan, Gokcen and Nair, Surag and Biancalani, Tommaso and Ji, Shuiwang and Regev, Aviv and Levine, Sergey and others},
  journal={arXiv preprint arXiv:2408.08252},
  year={2024}
}

@article{kadkhodaie2021stochastic,
  title={Stochastic solutions for linear inverse problems using the prior implicit in a denoiser},
  author={Kadkhodaie, Zahra and Simoncelli, Eero},
  journal={Advances in Neural Information Processing Systems},
  volume={34},
  pages={13242--13254},
  year={2021}
}

@misc{anonymous2025deradiff,
      title={DeRaDiff: Denoising Time Realignment of Diffusion Models}, 
      author={Ratnavibusena Don Shahain Manujith and Teoh Tze Tzun and Kenji Kawaguchi and Yang Zhang},
      year={2026},
      eprint={2601.20198},
      archivePrefix={arXiv},
      primaryClass={cs.LG},
      url={https://arxiv.org/abs/2601.20198}, 
}

@article{xu2023imagereward,
  title={Imagereward: Learning and evaluating human preferences for text-to-image generation},
  author={Xu, Jiazheng and Liu, Xiao and Wu, Yuchen and Tong, Yuxuan and Li, Qinkai and Ding, Ming and Tang, Jie and Dong, Yuxiao},
  journal={Advances in Neural Information Processing Systems},
  volume={36},
  pages={15903--15935},
  year={2023}
}

@inproceedings{ke2023vila,
  title={Vila: Learning image aesthetics from user comments with vision-language pretraining},
  author={Ke, Junjie and Ye, Keren and Yu, Jiahui and Wu, Yonghui and Milanfar, Peyman and Yang, Feng},
  booktitle={Proceedings of the IEEE/CVF Conference on Computer Vision and Pattern Recognition},
  pages={10041--10051},
  year={2023}
}
\bibliographystyle{plain}

%%%%%%%%%%%%%%%%%%%%%%%%%%%%%%%%%%%%%%%%%%%%%%%%%%%%%%%%%%%%

\newpage
\appendix
\section{Derivation of Lemma \ref{lemma:DPO}}\label{proof:dpo}
In this section, we provide the proof to derive the step-level DPO objective. For each $\mathbf{x}_{0:T}$, we have the following relationships between the reward and value functions:
\begin{flalign}
    &Q_{\pi_{\mathrm{pre}}}(s_t, a_t)=R(s_t, a_t)+V_{\pi_{\mathrm{pre}}}(s_{t+1}),\nonumber\\
     &A_{\pi_{\mathrm{pre}}}(s_t, a_t) = Q_{\pi_{\mathrm{pre}}}(s_t, a_t) - V_{\pi_{\mathrm{pre}}}(s_t).\nonumber
\end{flalign}
It then follows that
\begin{align}\label{eq:reward1}
r(\mathbf{x}_{0})
&= \sum_{t=0}^{T-1} R(s_t,a_t)\nonumber\\
&= \sum_{t=0}^{T-1} Q_{\pi_{\mathrm{pre}}}(s_t, a_t)-V_{\pi_{\mathrm{pre}}}(s_{t+1})\nonumber\\
&= V_{\pi_{\mathrm{pre}}}(s_0)+\sum_{t=0}^{T-1}
   \bigl(
     Q_{\pi_{\mathrm{pre}}}(s_t, a_t) - V_{\pi_{\mathrm{pre}}}(s_t)\bigr)-V_{\pi_{\mathrm{pre}}}(s_{T}
   ) \nonumber\\
&= V_{\pi_{\mathrm{pre}}}(s_0)
   + \sum_{t=0}^{T-1} A_{\pi_{\mathrm{pre}}}(s_t, a_t),
\end{align}
where the first equality is due to the definition of reward function $R$ and the last equality is due to $ V_{\pi_{\mathrm{pre}}}(s_T) = 0$.
According to the BT model, we have that
\begin{flalign}
    p_{BT}(\mathbf{x}_0^w\succ \mathbf{x}_0^l)=\sigma(r(\mathbf{x}_0^w)-r(\mathbf{x}_0^l)).\nonumber
\end{flalign}
Based on the BT model and Eq.~\eqref{eq:reward1}, we can write the loss as
\begin{flalign}
    -&\mathbb E_{(\mathbf{x}_0^w,\mathbf{x}_0^l)\sim \mathcal D}\log\sigma\left(r(\mathbf{x}_0^w)-r(\mathbf{x}_0^l)\right)\nonumber\\
    =&-\mathbb E_{(\mathbf{x}_0^w,\mathbf{x}_0^l)\sim \mathcal D}\log\sigma\left(\mathbb E_{\mathbf{x}_{0:T}^w\sim q(\mathbf{x}_{0:T}^w|\mathbf{x}_{0}^w), \mathbf{x}_{0:T}^l\sim q(\mathbf{x}_{0:T}^l|\mathbf{x}_{0}^l)}[r(\mathbf{x}_0^w)-r(\mathbf{x}_0^l)]\right)\nonumber\\
    =&-\mathbb E_{(\mathbf{x}_0^w,\mathbf{x}_0^l)\sim \mathcal D}\log\sigma\Bigg(\mathbb E_{\mathbf{x}_{0:T}^w\sim q(\mathbf{x}_{0:T}^w|\mathbf{x}_{0}^w), \mathbf{x}_{0:T}^l\sim q(\mathbf{x}_{0:T}^l|\mathbf{x}_{0}^l)}\Bigg[V_{\pi_{\mathrm{pre}}}(s_0^w)
   + \sum_{t=0}^{T-1} A_{\pi_{\mathrm{pre}}}(s_t^w, a_t^w)\nonumber\\
   &-V_{\pi_{\mathrm{pre}}}(s_0^l)
   - \sum_{t=0}^{T-1} A_{\pi_{\mathrm{pre}}}(s_t^l, a_t^l)\Bigg]\Bigg).
\end{flalign}
Similar to \cite{wallace2024diffusion}, we remove the $V_{\pi_{\mathrm{pre}}}(s_0^w), V_{\pi_{\mathrm{pre}}}(s_0^l)$ terms since they are constant and obtain the following objective
\begin{flalign}\label{eq:loss1}
    &-\mathbb E_{(\mathbf{x}_0^w,\mathbf{x}_0^l)\sim \mathcal D}\log\sigma\Bigg(\mathbb E_{\mathbf{x}_{0:T}^w\sim q(\mathbf{x}_{0:T}^w|\mathbf{x}_{0}^w), \mathbf{x}_{0:T}^l\sim q(\mathbf{x}_{0:T}^l|\mathbf{x}_{0}^l)}\Bigg[ \sum_{t=0}^{T-1} A_{\pi_{\mathrm{pre}}}(s_t^w, a_t^w)- \sum_{t=0}^{T-1} A_{\pi_{\mathrm{pre}}}(s_t^l, a_t^l)\Bigg]\Bigg)\nonumber\\
    &=-\mathbb E_{(\mathbf{x}_0^w,\mathbf{x}_0^l)\sim \mathcal D}\log\sigma\Bigg(T\mathbb E_{t\sim \mathcal{U}(0,T-1), \mathbf{x}_{0:T}^w\sim q(\mathbf{x}_{0:T}^w|\mathbf{x}_{0}^w), \mathbf{x}_{0:T}^l\sim q(\mathbf{x}_{0:T}^l|\mathbf{x}_{0}^l)}\Bigg[ A_{\pi_{\mathrm{pre}}}(s_t^w, a_t^w)-  A_{\pi_{\mathrm{pre}}}(s_t^l, a_t^l)\Bigg]\Bigg).
\end{flalign}
Moreover, based on the definition of value functions and Lemma \ref{lemma1}, we can show that 
\begin{flalign}
     A_{\pi_{\mathrm{pre}}}(s_t^w, a_t^w)&=Q_{\pi_{\mathrm{pre}}}(s_t^w, a_t^w)- V_{\pi_{\mathrm{pre}}}(s_t^w)\nonumber\\
     &=Q_{\pi_{\mathrm{pre}}}(s_t^w, a_t^w)- \mathbb E_{a\sim \pi_{\mathrm{pre}}}[Q_{\pi_{\mathrm{pre}}}(s_t^w, a)]\nonumber\\
     &=\lambda \log\left(\frac{p_\theta(\mathbf{x}_{T-t-1}^w|\mathbf{x}_{T-t}^w)}{p_{\text{pre}}(\mathbf{x}_{T-t-1}^w|\mathbf{x}_{T-t}^w)}\right)+\lambda\log(Z(\mathbf{x}_{T-t}^w)) \nonumber\\
     &-\mathbb E_{\mathbf{x}_{T-t-1}^w\sim \pi_{\text{pre}}}\left[\lambda \log\left(\frac{p_\theta(\mathbf{x}_{T-t-1}^w|\mathbf{x}_{T-t}^w)}{p_{\text{pre}}(\mathbf{x}_{T-t-1}^w|\mathbf{x}_{T-t}^w)}\right)\right]-\lambda\log(Z(\mathbf{x}_{T-t}^w)) \nonumber\\
     &=\lambda \log\left(\frac{p_\theta(\mathbf{x}_{T-t-1}^w|\mathbf{x}_{T-t}^w)}{p_{\text{pre}}(\mathbf{x}_{T-t-1}^w|\mathbf{x}_{T-t}^w)}\right)+\lambda\mathbb{D}_{\mathrm{KL}}(p_{\text{pre}}(\cdot|\mathbf{x}_{T-t}^w)||p_{\theta}(\cdot|\mathbf{x}_{T-t}^w)).
\end{flalign}
By Jensen’s inequality, the objective in Eq.~\eqref{eq:loss1} can be upper bounded by 
\begin{flalign}\label{eq:dpo3}
      -&\mathbb E_{(\mathbf{x}_0^w,\mathbf{x}_0^l)\sim \mathcal D}\log\sigma\Bigg(T\mathbb E_{t\sim \mathcal{U}(0,T-1), \mathbf{x}_{0:T}^w\sim q(\mathbf{x}_{0:T}^w|\mathbf{x}_{0}^w), \mathbf{x}_{0:T}^l\sim q(\mathbf{x}_{0:T}^l|\mathbf{x}_{0}^l)}\Bigg[ A_{\pi_{\mathrm{pre}}}(s_t^w, a_t^w)-  A_{\pi_{\mathrm{pre}}}(s_t^l, a_t^l)\Bigg]\Bigg)\nonumber\\
      \le&  -\mathbb E_{t\sim \mathcal{U}(0,T-1), (\mathbf{x}_0^w,\mathbf{x}_0^l)\sim \mathcal D, \mathbf{x}_{T-t}^w\sim q(\mathbf{x}_{T-t}^w|\mathbf{x}_{0}^w), \mathbf{x}_{T-t}^l\sim q(\mathbf{x}_{T-t}^l|\mathbf{x}_{0}^l) }\nonumber\\
      &\log\sigma\Bigg(T\mathbb E_{ \mathbf{x}_{T-t-1}^w\sim q(\mathbf{x}_{T-t-1}^w|\mathbf{x}_{T-t}^w, \mathbf{x}_{0}^w), \mathbf{x}_{T-t-1}^l\sim q(\mathbf{x}_{T-t-1}^l|\mathbf{x}_{T-t}^l, \mathbf{x}_{0}^l)}\Bigg[ A_{\pi_{\mathrm{pre}}}(s_t^w, a_t^w)-  A_{\pi_{\mathrm{pre}}}(s_t^l, a_t^l)\Bigg]\Bigg).
\end{flalign}
It can be further shown that
\begin{flalign}
    \mathbb E&_{\mathbf{x}_{T-t-1}^w\sim q(\mathbf{x}_{T-t-1}^w|\mathbf{x}_{T-t}^w, \mathbf{x}_{0}^w), \mathbf{x}_{T-t-1}^l\sim q(\mathbf{x}_{T-t-1}^l|\mathbf{x}_{T-t}^l, \mathbf{x}_{0}^l)}\Bigg[ A_{\pi_{\mathrm{pre}}}(s_t^w, a_t^w)\Bigg]\nonumber\\
    =&\lambda\mathbb{D}_{\mathrm{KL}}(p_{\text{pre}}(\cdot|\mathbf{x}_{T-t}^w)||p_{\theta}(\cdot|\mathbf{x}_{T-t}^w))\nonumber\\
    &+\lambda \mathbb E_{\mathbf{x}_{T-t-1}^w\sim q(\mathbf{x}_{T-t-1}^w|\mathbf{x}_{T-t}^w, \mathbf{x}_{0}^w), \mathbf{x}_{T-t-1}^l\sim q(\mathbf{x}_{T-t-1}^l|\mathbf{x}_{T-t}^l, \mathbf{x}_{0}^l)}\Bigg[\log\left(\frac{p_\theta(\mathbf{x}_{T-t-1}^w|\mathbf{x}_{T-t}^w)}{p_{\text{pre}}(\mathbf{x}_{T-t-1}^w|\mathbf{x}_{T-t}^w)}\right)\Bigg]\nonumber\\
    =&\lambda\mathbb{D}_{\mathrm{KL}}(p_{\text{pre}}(\cdot|\mathbf{x}_{T-t}^w)||p_{\theta}(\cdot|\mathbf{x}_{T-t}^w))-\lambda\mathbb{D}_{\mathrm{KL}}(q(\cdot|\mathbf{x}_{T-t}^w, \mathbf{x}_{0}^w )||\pi_{\theta}(\cdot|\mathbf{x}_{T-t}^w))+\lambda\mathbb{D}_{\mathrm{KL}}(q(\cdot|\mathbf{x}_{T-t}^w, \mathbf{x}_{0}^w)||p_{\text{pre}}(\cdot|\mathbf{x}_{T-t}^w)).
\end{flalign}
Applying the same analysis to $A_{\pi_{\mathrm{pre}}}(s_t^l, a_t^l)$ and plugging the results to Eq.~\eqref{eq:dpo3}, we then  can get the following objective:
\begin{flalign}
     -&\mathbb E_{t\sim \mathcal{U}(0,T-1), (\mathbf{x}_0^w,\mathbf{x}_0^l)\sim \mathcal D, \mathbf{x}_{T-t}^w\sim q(\mathbf{x}_{T-t}^w|\mathbf{x}_{0}^w), \mathbf{x}_{T-t}^l\sim q(\mathbf{x}_{T-t}^l|\mathbf{x}_{0}^l) }\nonumber\\
      &\log\sigma\Bigg(T\mathbb E_{ \mathbf{x}_{T-t-1}^w\sim q(\mathbf{x}_{T-t-1}^w|\mathbf{x}_{T-t}^w, \mathbf{x}_{0}^w), \mathbf{x}_{T-t-1}^l\sim q(\mathbf{x}_{T-t-1}^l|\mathbf{x}_{T-t}^l, \mathbf{x}_{0}^l)}\Bigg[ A_{\pi_{\mathrm{pre}}}(s_t^w, a_t^w)-  A_{\pi_{\mathrm{pre}}}(s_t^l, a_t^l)\Bigg]\Bigg)\nonumber\\
      =&\mathbb E_{t\sim \mathcal{U}(0,T-1), (\mathbf{x}_0^w,\mathbf{x}_0^l)\sim \mathcal D, \mathbf{x}_{T-t}^w\sim q(\mathbf{x}_{T-t}^w|\mathbf{x}_{0}^w), \mathbf{x}_{T-t}^l\sim q(\mathbf{x}_{T-t}^l|\mathbf{x}_{0}^l) }\nonumber\\
      &\log\sigma\Bigg(T\lambda\Bigg(-\mathbb{D}_{\mathrm{KL}}(q(\cdot|\mathbf{x}_{T-t}^w, \mathbf{x}_{0}^w )||\pi_{\theta}(\cdot|\mathbf{x}_{T-t}^w))+\mathbb{D}_{\mathrm{KL}}(q(\cdot|\mathbf{x}_{T-t}^l, \mathbf{x}_{0}^l )||\pi_{\theta}(\cdot|\mathbf{x}_{T-t}^l))\nonumber\\
      &+\mathbb{D}_{\mathrm{KL}}(q(\cdot|\mathbf{x}_{T-t}^w, \mathbf{x}_{0}^w)||p_{\text{pre}}(\cdot|\mathbf{x}_{T-t}^w))-\mathbb{D}_{\mathrm{KL}}(q(\cdot|\mathbf{x}_{T-t}^l, \mathbf{x}_{0}^l)||p_{\text{pre}}(\cdot|\mathbf{x}_{T-t}^l))\nonumber\\
      &+\mathbb{D}_{\mathrm{KL}}(p_{\text{pre}}(\cdot|\mathbf{x}_{T-t}^w)||p_{\theta}(\cdot|\mathbf{x}_{T-t}^w))-\mathbb{D}_{\mathrm{KL}}(p_{\text{pre}}(\cdot|\mathbf{x}_{T-t}^l)||p_{\theta}(\cdot|\mathbf{x}_{T-t}^l))\Bigg)\Bigg)
\end{flalign}
By the Gaussian parameterization of the reverse process \cite{wallace2024diffusion}, this objective can be simplified to 
\begin{flalign}
L_{SDPO}=&-\mathbb{E}_{(\mathbf{x}_0^w,\mathbf{x}_0^l)\sim \mathcal{D},\, t\sim \mathcal{U}(1,T),\,
      \mathbf{x}_t^w\sim q(\mathbf{x}_t^w \mid \mathbf{x}_0^w),\, \mathbf{x}_t^l\sim q(\mathbf{x}_t^l \mid \mathbf{x}_0^l)}\nonumber\\
     & 
       \Big[
      \log \sigma\Big( -\lambda T \omega (\lambda_t)\Big(
          \|\epsilon^w - \epsilon_\theta(\mathbf{x}_t^w,t)\|_2^2
          - \|\epsilon^w - \epsilon_{\mathrm{pre}}(\mathbf{x}_t^w,t)\|_2^2 \nonumber\\
&
          - \big(
              \|\epsilon^l - \epsilon_\theta(\mathbf{x}_t^l,t)\|_2^2
              - \|\epsilon^l - \epsilon_{\mathrm{pre}}(\mathbf{x}_t^l,t)\|_2^2
            \big)\nonumber\\
 &   - \big(
              \|\epsilon_\theta(\mathbf{x}_t^w,t)-\epsilon_{\mathrm{pre}}(\mathbf{x}_t^w,t)\|_2^2
              - \|\epsilon_{\theta}(\mathbf{x}_t^l,t) - \epsilon_{\mathrm{pre}}(\mathbf{x}_t^l,t)\|_2^2
            \big)
        \Big)
      \Big)
      \Big], 
\end{flalign}
where $\mathbf{x}_{0}^w$ and $\mathbf{x}_{0}^l$ denote the winning and losing samples, respectively and  we have that
\[
\mathbf{x}_t^{*} = \sqrt{\bar{\alpha}_t}\,\mathbf{x}_0^{*}
+ \sqrt{1-\bar{\alpha}_t}\,\epsilon^{*}, 
\qquad
\epsilon^{*} \sim \mathcal{N}(\mathbf{0},\mathbf{I}),
\]
which is drawn from \(q(\mathbf{x}_t^{*} \mid \mathbf{x}_0^{*})\), and the quantity \(\lambda_t = \bar{\alpha}_t / (1-\bar{\alpha}_t)\) is the signal-to-noise ratio. The function $\omega(\lambda_t)$ is a weighting function (often taken to be constant in practice~\cite{ho2020denoising,kingma2021variational}). $\epsilon_*(\mathbf{x}_t,t)$ is a function approximator that predicts $\epsilon$ from $\mathbf{x}_t$ such that 
\begin{flalign}
    \mu_*(\mathbf{x}_t,t)
    = \frac{1}{\sqrt{\alpha_t}}
      \left(
        \mathbf{x}_t
        - \frac{\beta_t}{\sqrt{1-\bar{\alpha}_t}}
          \,\epsilon_*(\mathbf{x}_t,t)
      \right).
\end{flalign}

\section{Proof of Lemma \ref{lemma1}}\label{proof:lemma}
\begin{proof}
We can first define that
\begin{align}
\mathcal{J}(\pi_\theta; s_t)
&= \mathbb{E}_{z \sim \pi_\theta(\cdot \mid s_t)}
   \left[
     Q_{\pi_{\mathrm{pre}}}(s_t, z) - V_{\pi_{\mathrm{pre}}}(s_t)
     - \lambda
       \log
       \frac{
         \pi_\theta(z \mid s_t)
       }{
         \pi_{\mathrm{pre}}(z \mid s_t)
       }
   \right] \nonumber \\
&= \lambda
   \mathbb{E}_{z \sim \pi_\theta(\cdot \mid s_t)}
   \left[
     \log
     \frac{
       \pi_{\mathrm{pre}}(z \mid s_t)
       \exp\bigl(Q_{\pi_{\mathrm{pre}}}(s_t, z)/\lambda\bigr)
     }{
       \pi_\theta(z \mid s_t)
     }
   \right]
   - V_{\pi_{\mathrm{pre}}}(s_t),\nonumber
\end{align}
where we have that $
    Z(s_t)
  = \int
      \pi_{\mathrm{ref}}(z \mid s_t)
      \exp\bigl(Q_{\pi_{\mathrm{ref}}}(s_t, z)/\lambda\bigr)
    \, dz.
$
The optimal solution of Eq.~\eqref{eq:rlhf3} can be obtained by optimizing $\mathcal{J}(\pi_\theta; s_t)$ for each $s_t$. We can further rewrite $\mathcal{J}(\pi_\theta; s_t)$ as:
\begin{align}
\mathcal{J}(\pi_\theta; s_t)
&= - \lambda
   \mathbb D_{\mathrm{KL}}\left(
     \pi_\theta(\cdot \mid s_t)
     \,\middle\|\,
     \frac{
       \pi_{\mathrm{pre}}(\cdot \mid s_t)
       \exp\bigl(Q_{\pi_{\mathrm{pre}}}(s_t, \cdot)/\lambda\bigr)
     }{
       Z(s_t)
     }
   \right)
   - V_{\pi_{\mathrm{pre}}}(s_t)
   + \lambda \log Z(s_t).
\end{align}
Thus, we can find the optimal solution:
\begin{flalign}
\pi_\theta^\star(z \mid s_t)
  = \frac{
      \pi_{\mathrm{ref}}(z \mid s_t)
      \exp\bigl(Q_{\pi_{\mathrm{pre}}}(s_t, z)/\lambda\bigr)
    }{
      Z(s_t)
    }.
\end{flalign}
This completes the proof.
\end{proof}

\section{Proof of Theorem \ref{theorem1}}\label{proof:the}
%\paragraph{Gaussian case and closed-form parameters.}
\begin{proof}
Based on Lemma \ref{lemma1}, we can show that for each $1\le i\le M$, we have that
\begin{flalign}
\pi_i(z \mid s_t)
  = \frac{
      \pi_{\mathrm{pre}}(z \mid s_t)
      \exp\bigl(Q_{\pi_{\mathrm{pre}}}^i(s_t, z)/\lambda\bigr)
    }{
      Z_i(s_t)
    },\nonumber
\end{flalign}
where $Q_{\pi_{\mathrm{pre}}}^i$ is the state-action value function with reward $r_i$ and $
    Z_i(s_t)
  = \int
      \pi_{\mathrm{ref}}(z \mid s_t)
      \exp\bigl(Q^i_{\pi_{\mathrm{pre}}}(s_t, z)/\lambda\bigr)
    \, dz$. Moreover, we can further show that
\begin{flalign}\label{eq:piw}
\pi_w(z \mid s_t)
  = \frac{
      \pi_{\mathrm{pre}}(z \mid s_t)
      \exp\bigl(Q_{\pi_{\mathrm{pre}}}^w(s_t, z)/\lambda\bigr)
    }{
      Z_w(s_t)
    },
\end{flalign}
where $Q_{\pi_{\mathrm{pre}}}^w$ is the state-action value function with reward $r_w$ and $
    Z_w(s_t)
  = \int
      \pi_{\mathrm{ref}}(z \mid s_t)
      \exp\bigl(Q^w_{\pi_{\mathrm{pre}}}(s_t, z)/\lambda\bigr)
    \, dz.$
For the state-action function, we can show that 
\begin{flalign}
    Q^w_{\pi_{\mathrm{pre}}}(s_t, z)&=\mathbb{E}_{\pi_{\mathrm{pre}}}\bigl[ r^w(s_T) \mid s_t, a_t \bigr]\nonumber\\
    &=\mathbb{E}_{\pi_{\mathrm{pre}}}\left[ \sum_{i=1}^Mw_ir_i(s_T) \mid s_t, a_t \right]\nonumber\\
    &=\sum_{i=1}^Mw_i\mathbb{E}_{\pi_{\mathrm{pre}}}\bigl[ r_i(s_T) \mid s_t, a_t \bigr]\nonumber\\
    &=\sum_{i=1}^Mw_iQ^i_{\pi_{\mathrm{pre}}}(s_t, z).
\end{flalign}
As a result, we can show that
\begin{flalign}
    \exp\bigl(Q_{\pi_{\mathrm{pre}}}^w(s_t, z)/\lambda\bigr)&=\exp\left(\sum_{i=1}^Mw_iQ^i_{\pi_{\mathrm{pre}}}(s_t, z)/\lambda\right)\nonumber\\
    &=\prod_{i=1}^M\exp\bigl(w_iQ^i_{\pi_{\mathrm{pre}}}(s_t, z)/\lambda\bigr)\nonumber\\
    &=\prod_{i=1}^M\exp\bigl(Q^i_{\pi_{\mathrm{pre}}}(s_t, z)/\lambda\bigr)^{w_i}\nonumber\\
    &=\frac{\prod_{i=1}^M \pi_i^{w_i}(z \mid s_t) Z_i^{w_i}(s_t)}{{\pi_{\mathrm{pre}}}}.\nonumber
\end{flalign}
Plugging the above equation into Eq.~\eqref{eq:piw}, it can be shown that 
\begin{flalign}
    \pi_w(z \mid s_t)=\frac{ \prod_{i=1}^M \pi_i^{w_i}(z \mid s_t)}{\int \prod_{i=1}^M \pi_i^{w_i}(z' \mid s_t) dz'}.\nonumber
\end{flalign}
Based on the definition that $\pi_i(a_t|s_t)=p_i(\mathbf{x}_{T-t-1}|\mathbf{x}_{T-t})$ for $s_t=\mathbf{x}_{T-t}$ and $a_t=\mathbf{x}_{T-t-1}$, we know that 
\begin{flalign}
{p}_w(\mathbf{x}_{t-1}\mid \mathbf{x}_{t})
  =
  \frac{
    \prod_{i=1}^M {p}_i^{w_i}(\mathbf{x}_{t-1}\mid \mathbf{x}_{t})
  }{
    \displaystyle
    \int
    \prod_{i=1}^M {p}_i^{w_i}(\mathbf{x}'_{t-1}\mid \mathbf{x}_{t}) \, d\mathbf{x}'_{t-1}\nonumber
  },
\end{flalign}
where each ${p}_i(\mathbf{x}_{t-1}\mid \mathbf{x}_{t})  $ follows the Gaussian distribution with mean and variance $(\mu_i, \sigma_i^2) $. Thus, ${p}_w(\mathbf{x}_{t-1}\mid \mathbf{x}_{t})$ is also a Gaussian distribution, and ${p}_w(\mathbf{x}_{t-1}\mid \mathbf{x}_{t}) \propto \exp\left(-\sum_{i=1}^M\frac{w_i}{2\sigma_i^2}\|\mathbf{x}_{t-1}-\mu_i\|^2\right)$. We can then derive its mean and variance as: 
\begin{flalign}
    &\sigma_w^2= \left(\sum_{i=1}^M \frac{w_i}{\sigma_i^2}\right)^{-1},\nonumber\\
    &\mu_{w}=\sigma_w^2\sum_{i=1}^M \frac{w_i}{\sigma_i^2}\,\mu_i.\nonumber
\end{flalign}
This completes the proof.
\end{proof}

\section{Approximated Errors in Existing Denoising-Time Works}\label{approxi}
\textbf{DB-MPA~\cite{cheng2025diffusion}} considers a stochastic differential equation (SDE) formulation:
\[
\mathrm{d}\mathbf{x}_t = f^{\mathrm{pre}}(\mathbf{x}_t,t)\,\mathrm{d}t + \sigma(t)\,\mathrm{d}w_t,\qquad \forall\, t\in[T,0],
\]
where $\sigma(t), \beta(t)$ represents two fixed scalar functions, $w_t$ is a standard Wiener
process,  $p_t$ is a marginal distribution and  $f^{\mathrm{pre}}(\mathbf{x}_t,t)$ denotes the term $$ f^{\mathrm{pre}}(\mathbf{x}_t,t)=
-\tfrac{1}{2}\beta(t)\,\mathbf{x}_t \;-\; \beta(t)\,\nabla_{\mathbf{x}_t}\log p_t(\mathbf{x}_t).
$$
To maximize Eq.~\eqref{eq:rlhf1} with reward $r$, the SDE becomes 
\[
\mathrm{d}\mathbf{x}_t = f^{(r,\lambda)}(\mathbf{x}_t,t)\,\mathrm{d}t + \sigma(t)\,\mathrm{d}w_t,\qquad \forall\, t\in[T,0],
\]
where 
\[
f^{(r,\lambda)}(\mathbf{x}_t,t) = f^{\mathrm{pre}}(\mathbf{x}_t,t) + u^{(r,\lambda)}(\mathbf{x}_t,t),
\]
and
\[
u^{(r,\lambda)}(\mathbf{x}_t,t)
= \nabla_{\mathbf{x}_t} \log \mathbb{E}_{\mathbf{x}_0 \sim p^{\mathrm{pre}}_{0|t}(\cdot \mid \mathbf{x}_t)}
\left[ \exp\left( \frac{r(\mathbf{x}_0)}{\lambda} \right) \right].
\]
Let $\Delta^{(r_i,\lambda)}(\mathbf{x},t)=u^{(r,\lambda)}(\mathbf{x}_t,t)- \nabla_{\mathbf{x}_t}  \mathbb{E}_{\mathbf{x}_0 \sim p^{\mathrm{pre}}_{0|t}(\cdot \mid \mathbf{x}_t)}
\left[  \frac{r(\mathbf{x}_0)}{\lambda} \right]$, for each $r^w=\sum_{i=1}^Mw_ir_i$, it can be shown  that with the identical variance, 
\begin{flalign}\label{eq:DB}
    f^{(r^w,\lambda)}(\mathbf{x}_t,t)
= \sum_{i=1}^M w_i\, f^{(r_i,\lambda)}(\mathbf{x}_t,t)
+ \Bigl(\Delta^{(r^w,\lambda)}(\mathbf{x}_t,t) - \sum_{i=1}^M w_i\,\Delta^{(r_i,\lambda)}(\mathbf{x}_t,t)\Bigr).
\end{flalign}
To get the denoising-time solution, DB-MPA ignores the last two terms in Eq.~\eqref{eq:DB} and gets $f^{(r^w,\lambda)}(\mathbf{x}_t,t)\approx  \sum_{i=1}^M w_i\, f^{(r_i,\lambda)}(\mathbf{x}_t,t).$

\textbf{DERADIFF~\cite{anonymous2025deradiff}} studies the RL fine-tuning problem in Eq.~\eqref{eq:rlhf1}. Similar to our Lemma \ref{lemma1} and Theorem \ref{theorem1}, the optimal solution should satisfy that 
\begin{flalign}
    p_w(\mathbf{x}_0)=\frac{ \prod_{i=1}^M p_i^{w_i}(\mathbf{x}_0)}{\int \prod_{i=1}^M p_i^{w_i}(\mathbf{x}_0') d\mathbf{x}_0'}.
\end{flalign}
However, it is hard to track the marginal distribution of $\mathbf{x}_0$, thus DERADIFF conducts a stepwise approximation for any $t$ without any theoretical justification:
\begin{flalign}
    p_w(\mathbf{x}_{t-1}|\mathbf{x}_t)=\frac{ \prod_{i=1}^M p_i^{w_i}(\mathbf{x}_{t-1}|\mathbf{x}_t)}{\int \prod_{i=1}^M p_i^{w_i}(\mathbf{x}_{t-1}'|\mathbf{x}_t) d\mathbf{x}_{t-1}'}.
\end{flalign}

\section{Experiment details}\label{sec:detail}
\subsection{Prompts}
For completeness, we provide the prompt dataset \cite{cheng2025diffusion} in Table \ref{tab:prompts}, which is a subset of DrawBench \cite{saharia2022photorealistic} restricted to the ``color'' category. To obtain test prompts that do not appear in the training data, GPT-4 is used to synthesize novel color–object and object–object combinations derived from the training set. 
\begin{table}[htbp]
\centering
\small
\renewcommand{\arraystretch}{1.05}
\begin{tabular}{|p{0.46\textwidth}|p{0.46\textwidth}|}
\hline
\textbf{Training Prompts} & \textbf{Test Prompts} \\
\hline
\begin{minipage}[t]{\linewidth}\raggedright
A red colored car.\\
A black colored car.\\
A pink colored car.\\
A black colored dog.\\
A red colored dog.\\
A blue colored dog.\\
A green colored banana.\\
A red colored banana.\\
A black colored banana.\\
A white colored sandwich.\\
A black colored sandwich.\\
An orange colored sandwich.\\
A pink colored giraffe.\\
A yellow colored giraffe.\\
A brown colored giraffe.\\
A red car and a white sheep.\\
A blue bird and a brown bear.\\
A green apple and a black backpack.\\
A green cup and a blue cell phone.\\
A yellow book and a red vase.\\
A white car and a red sheep.\\
A brown bird and a blue bear.\\
A black apple and a green backpack.\\
A blue cup and a green cell phone.\\
A red book and a yellow vase.
\end{minipage}
&
\begin{minipage}[t]{\linewidth}\raggedright
A white colored dog.\\
A purple colored dog.\\
A yellow colored dog.\\
A green colored apple.\\
A black colored apple.\\
A blue colored apple.\\
A purple colored apple.\\
A pink colored banana.\\
A pink colored cup.\\
A purple colored sandwich.\\
A green colored giraffe.\\
A blue colored backpack.\\
A blue car and a pink sheep.\\
A red apple and a purple backpack.\\
A pink car and a yellow sheep.\\
A black cup and a yellow cell phone.\\
A blue car and a red giraffe.\\
A yellow bird and a purple sheep.\\
A pink car and a green bear.\\
A purple hat and a black clock.\\
A black chair and a red table.\\
A red car and a blue bird.\\
A green car and a yellow banana.\\
A pink vase and a red apple.\\
A blue book and an orange colored sandwich.
\end{minipage}
\\
\hline
\end{tabular}
\caption{Training and test prompt datasets used in the experiments.}
\label{tab:prompts}
\end{table}
\subsection{Training and Evaluation Details}
For the RL fine-tuning, we followed the DPOK \cite{fan2023dpok} method. For each diffusion model, we used one H100 GPU for fine-tuning, where the batch size was set to $2$ and gradient accumulation was $12$. We used the same setup \cite{cheng2025diffusion} that the learning rate of AdamW was set to $1\times 10^{-5}$ and the LoRA rank was set to $4$. For the policy update, we trained the model for $8000$ epochs, and other settings were the same as \cite{fan2023dpok}: we set the clipping ratio to $1\times 10^{-4}$ and ran $5$ policy gradient steps and value function updates each iteration. The $\eta$ of the denoising scheduler was set to $1.0$ for the VILA model and $0.8$ for the ImageReward model. For CoDe, we set the number of particles for the search to $20$ and the lookahead steps to $5$.

\subsection{Generated Images}
In the following Figs. \ref{fig:e1}-\ref{fig:e4}, we provide some generated image samples from SD, RS and our method with different $w$. The left image is generated by Stable Diffusion using the prompt shown in the lower-left corner. The top row shows images generated by our proposed MSDDA, while the bottom row shows images generated by RS.
\begin{figure}
  \centering
\includegraphics[width=0.6\textwidth]{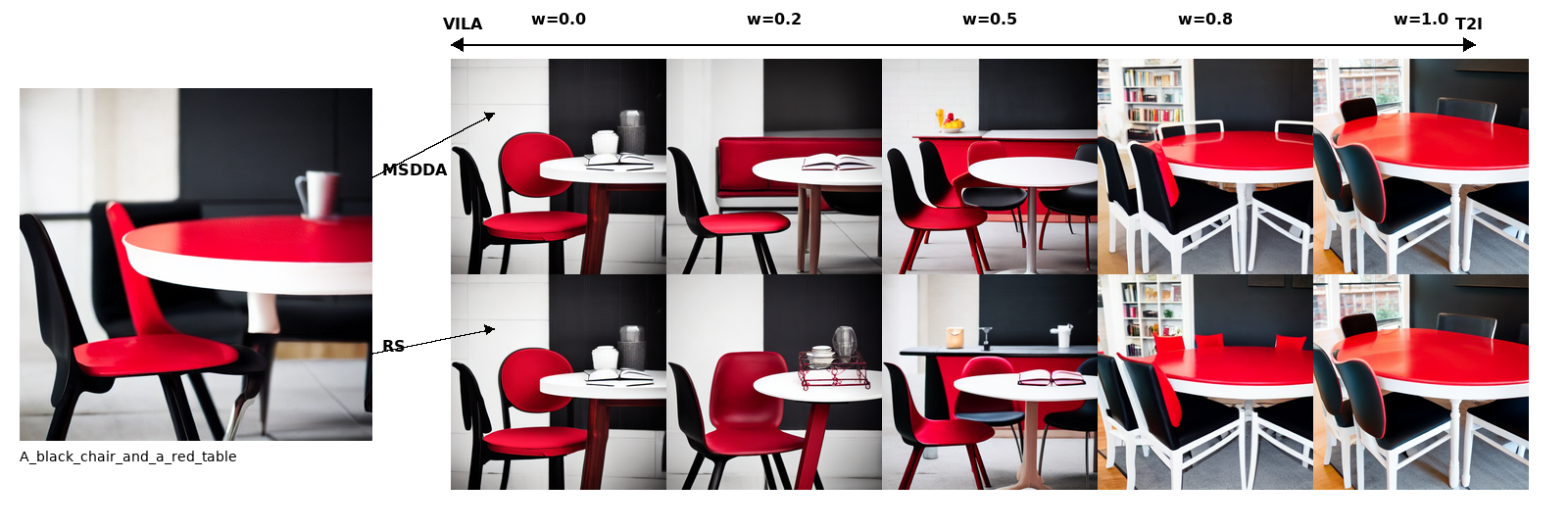}
\includegraphics[width=0.6\textwidth]{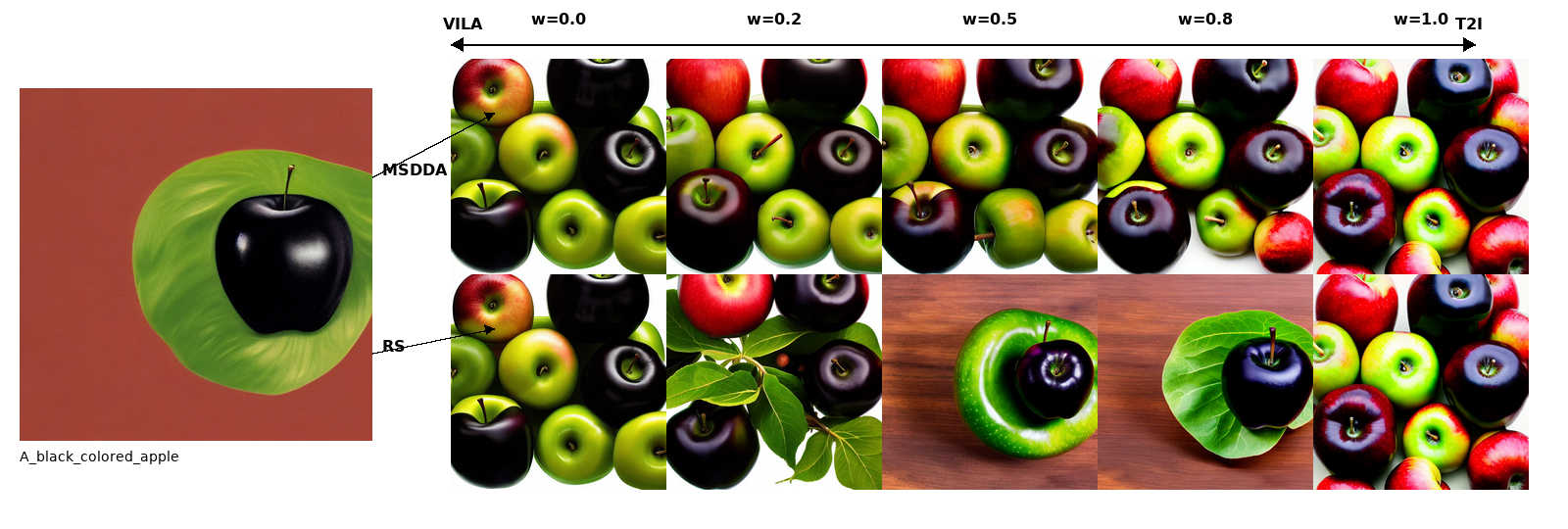}
\includegraphics[width=0.6\textwidth]{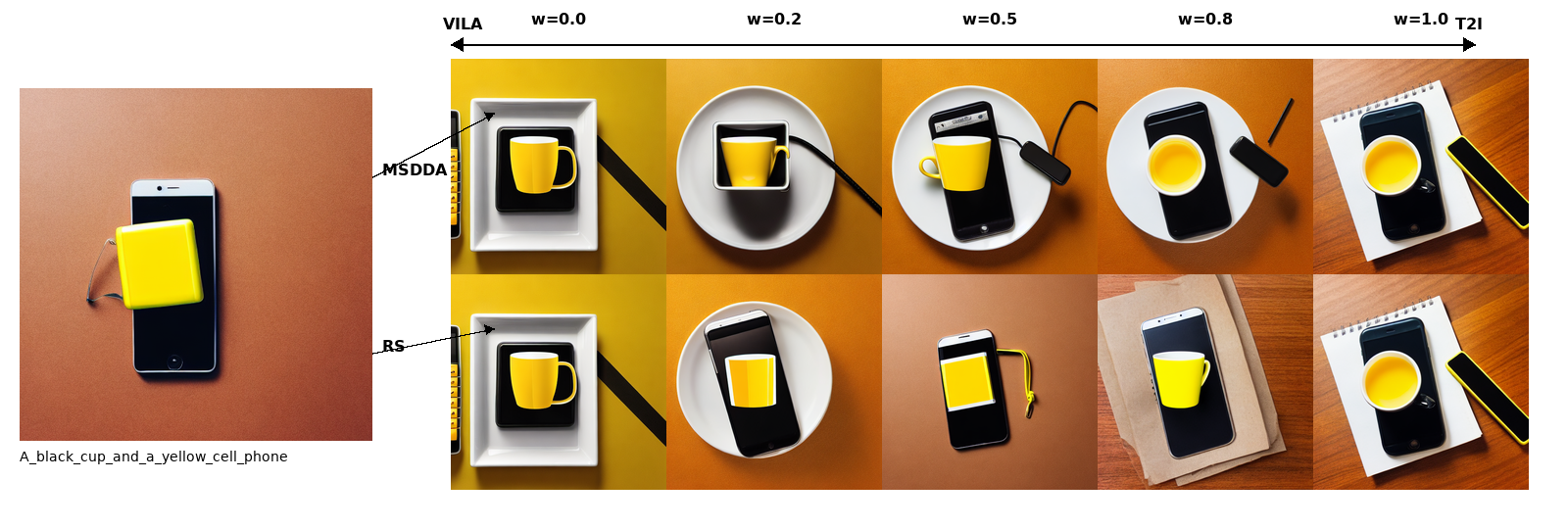}
\includegraphics[width=0.6\textwidth]{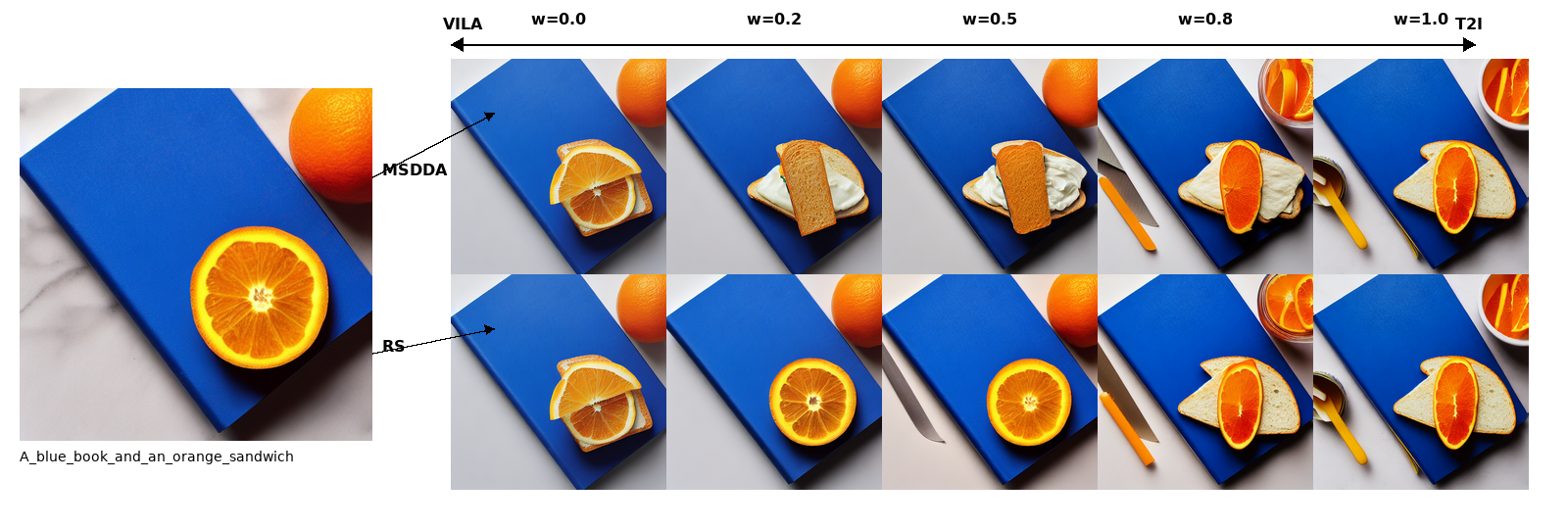}
  \caption{Image samples of our proposed method and RS. }\label{fig:e1}
\end{figure}

\begin{figure}
  \centering
\includegraphics[width=0.6\textwidth]{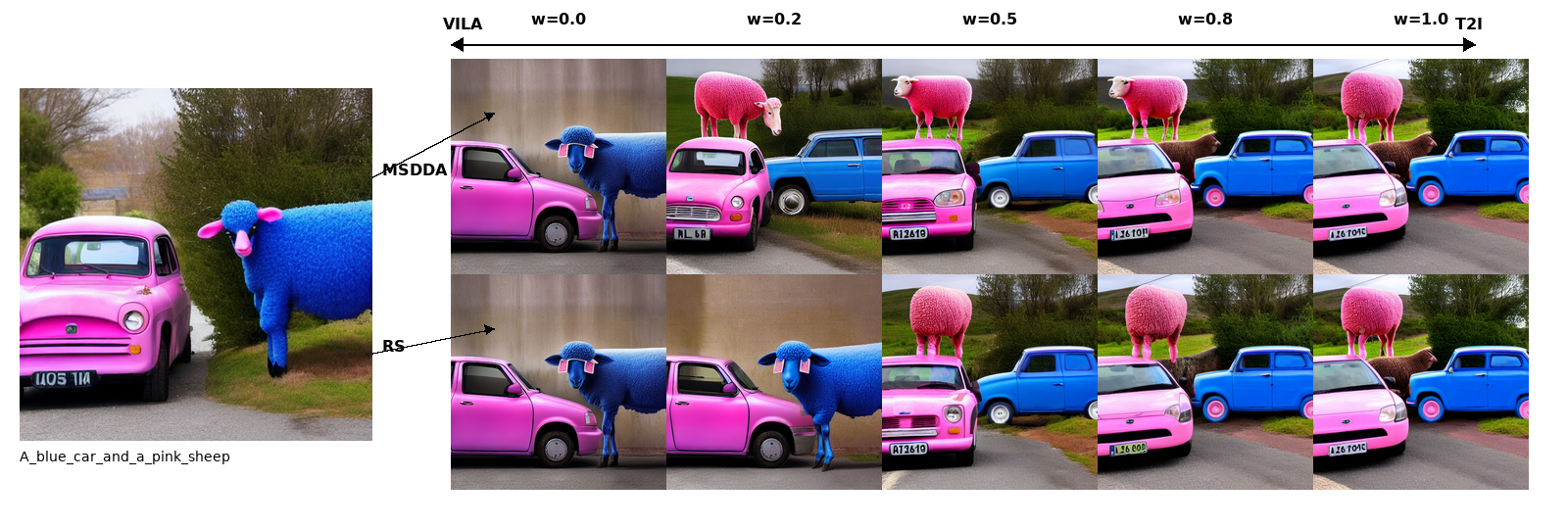}
\includegraphics[width=0.6\textwidth]{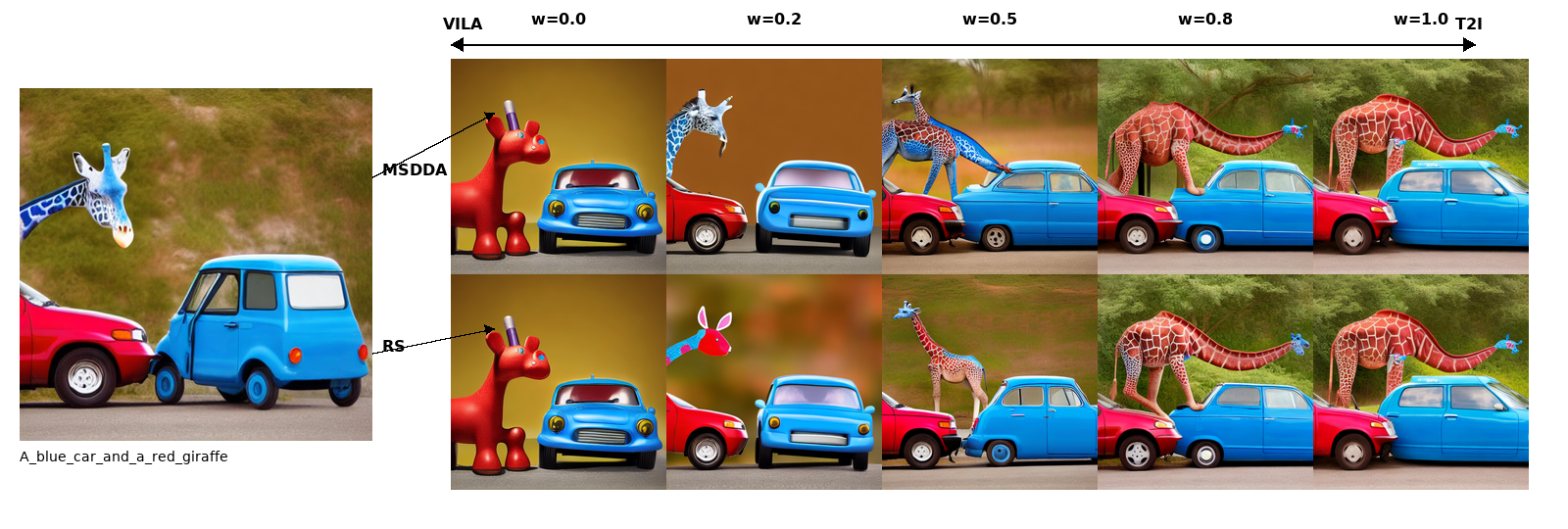}
\includegraphics[width=0.6\textwidth]{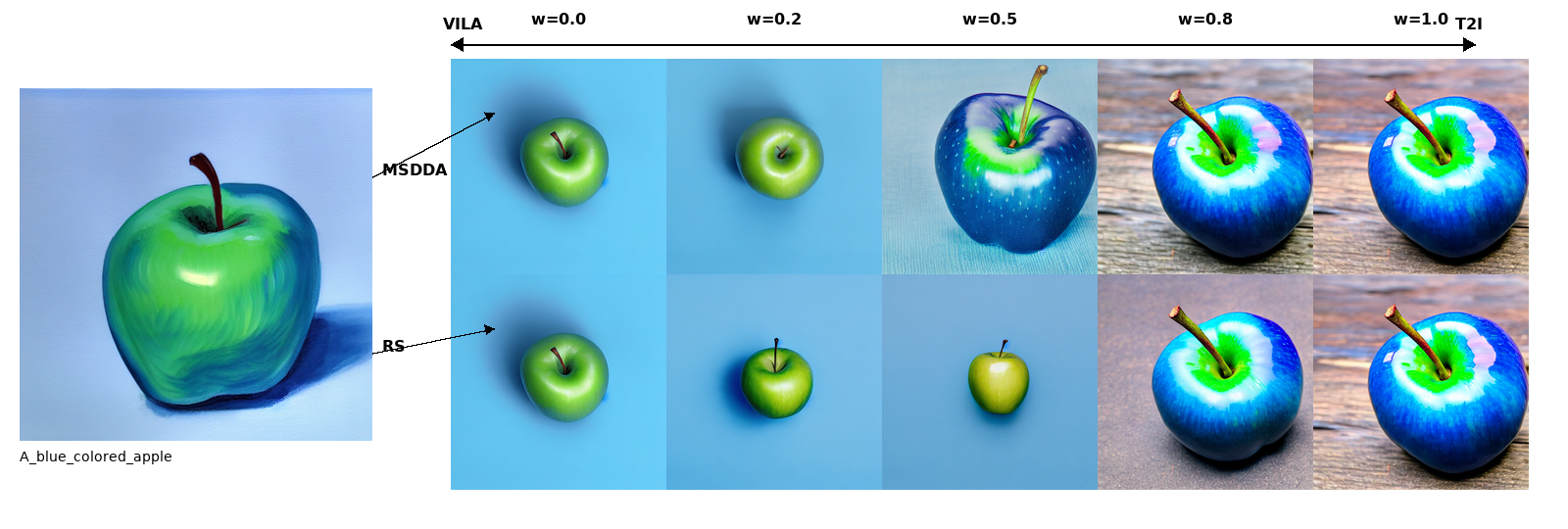}
\includegraphics[width=0.6\textwidth]{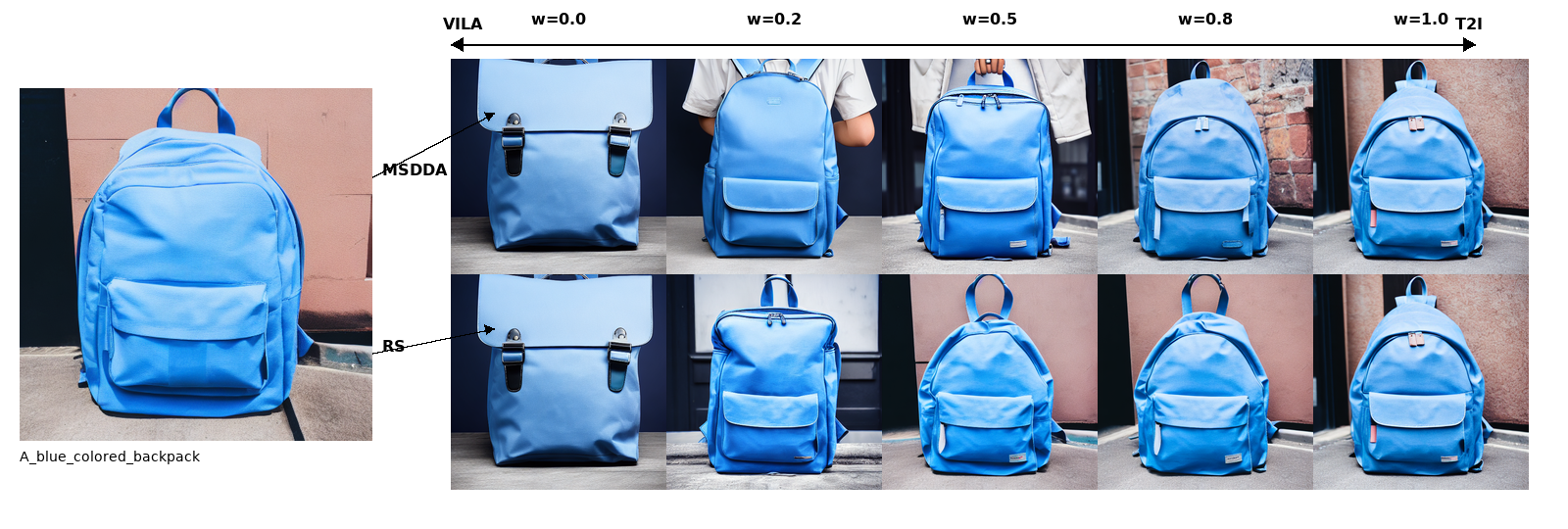}
\includegraphics[width=0.6\textwidth]{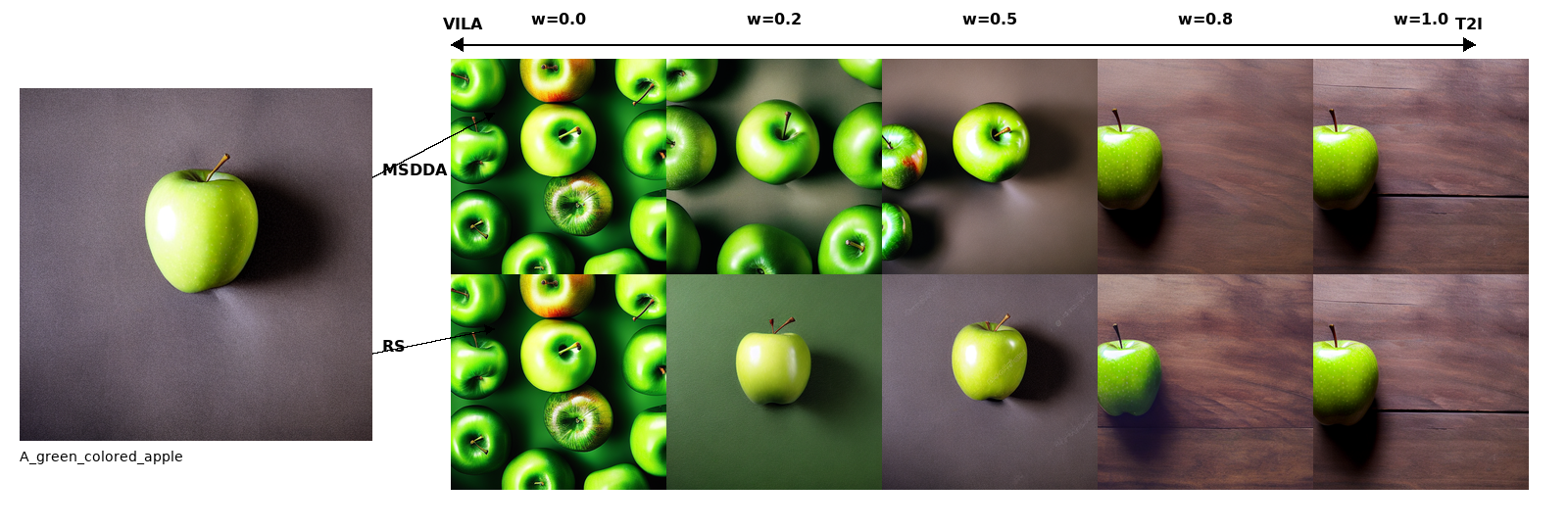}
  \caption{Image samples of our proposed method and RS. }\label{fig:e2}
\end{figure}

\begin{figure}
  \centering
\includegraphics[width=0.6\textwidth]{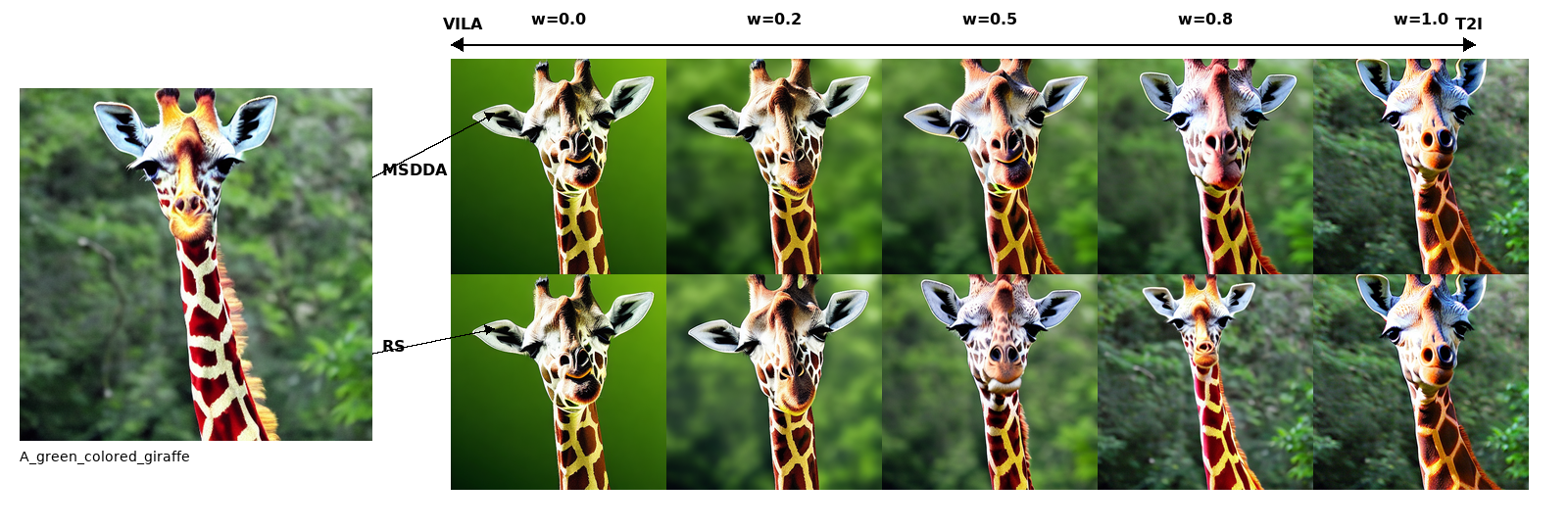}
\includegraphics[width=0.6\textwidth]{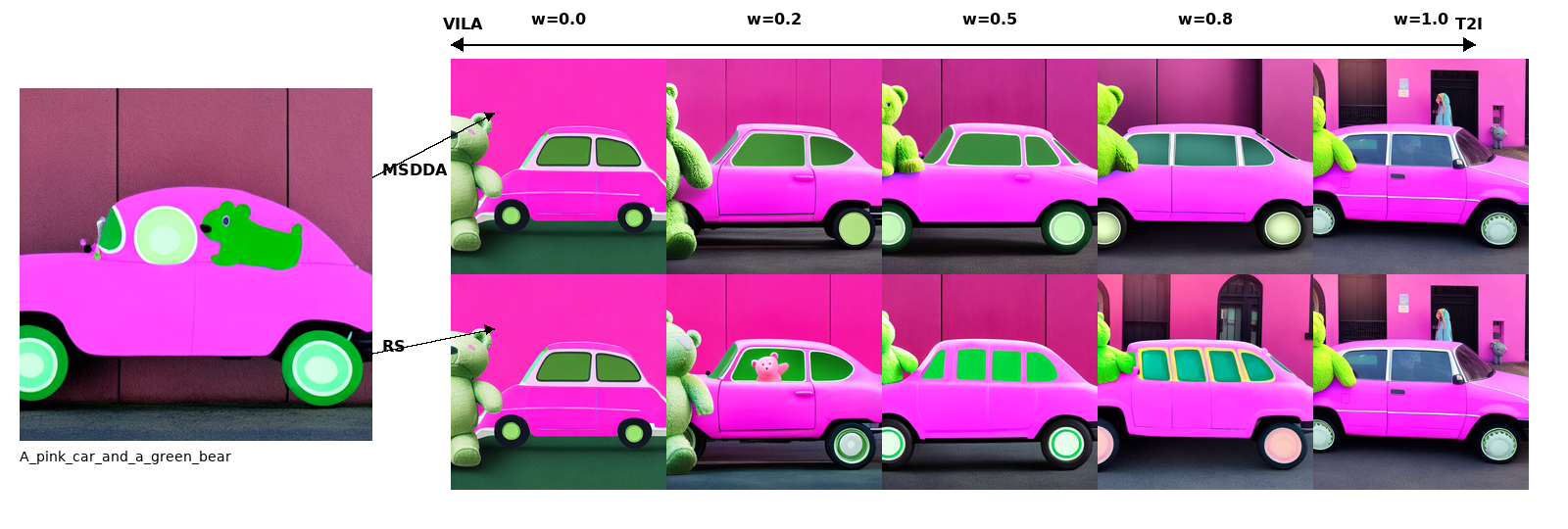}
\includegraphics[width=0.6\textwidth]{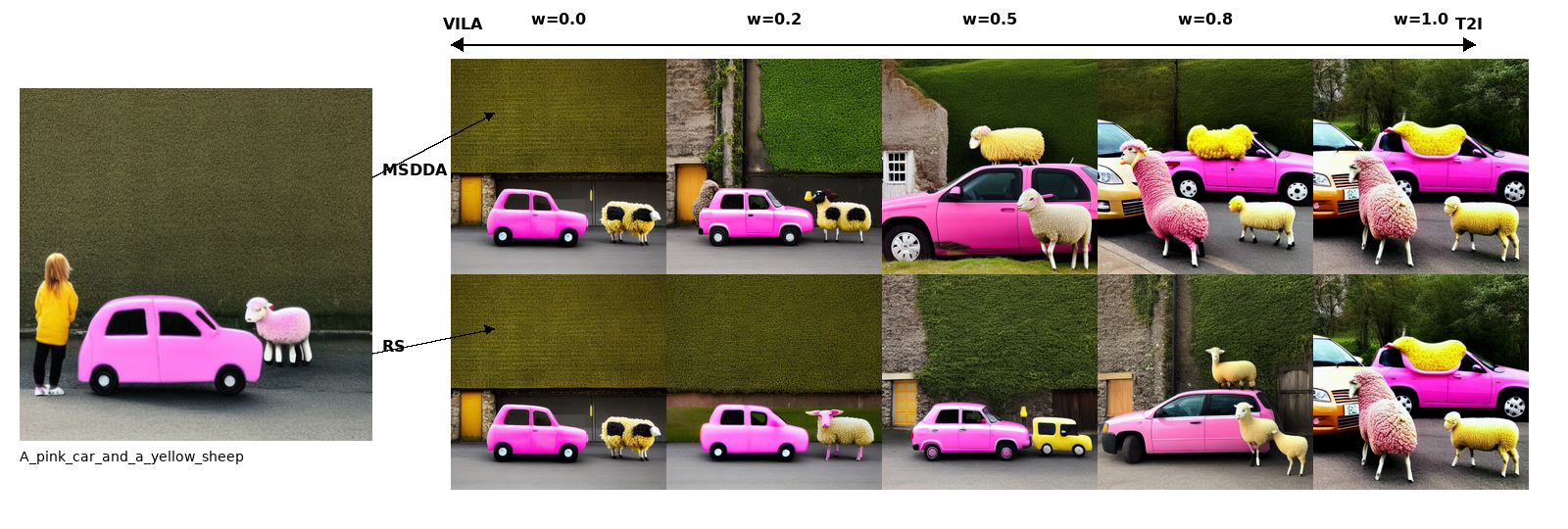}
\includegraphics[width=0.6\textwidth]{A_pink_colored_banana3.png}
\includegraphics[width=0.6\textwidth]{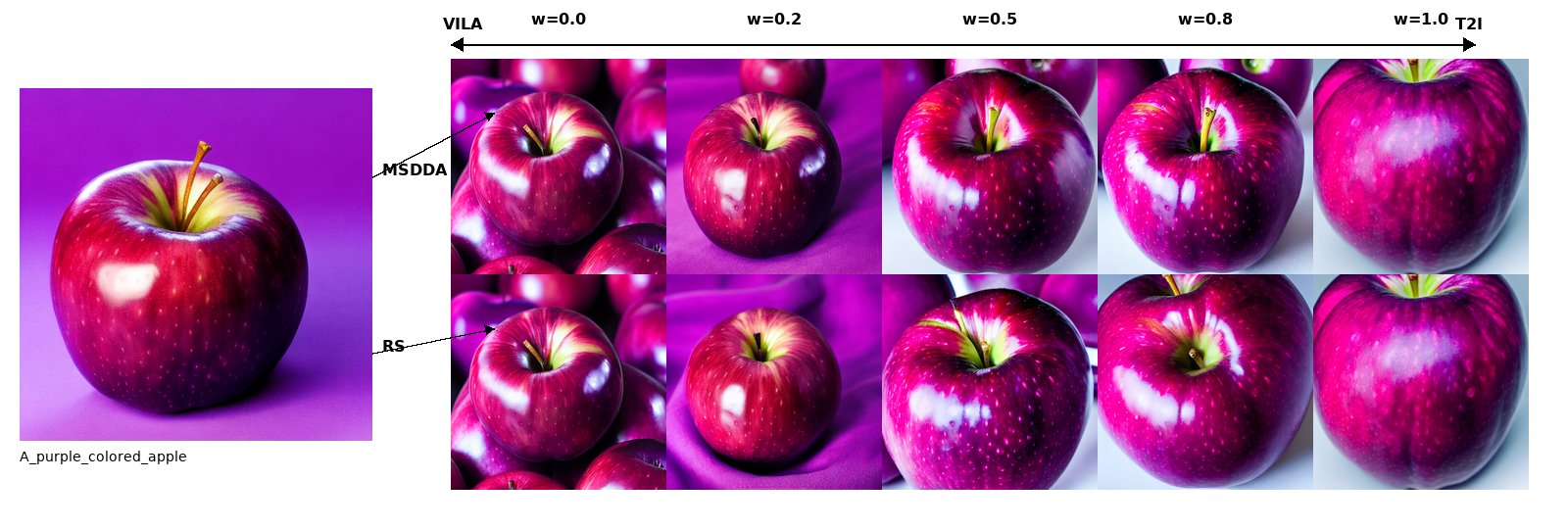}
  \caption{Image samples of our proposed method and RS. }\label{fig:e3}
\end{figure}

\begin{figure}
  \centering
\includegraphics[width=0.6\textwidth]{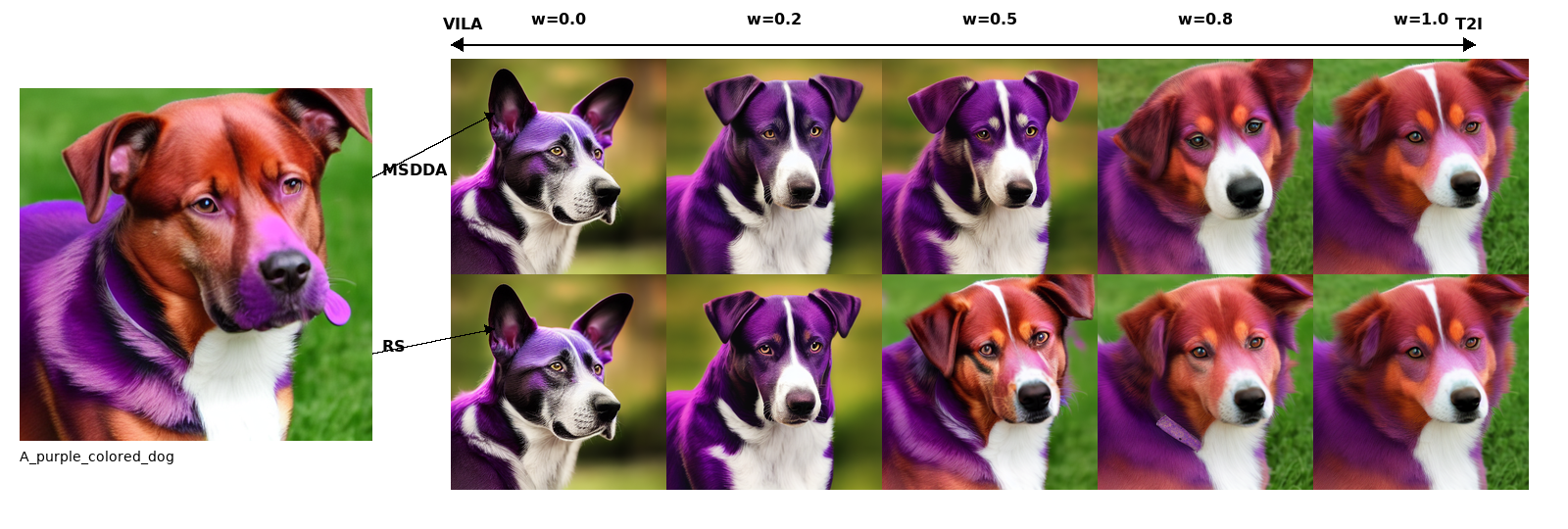}
\includegraphics[width=0.6\textwidth]{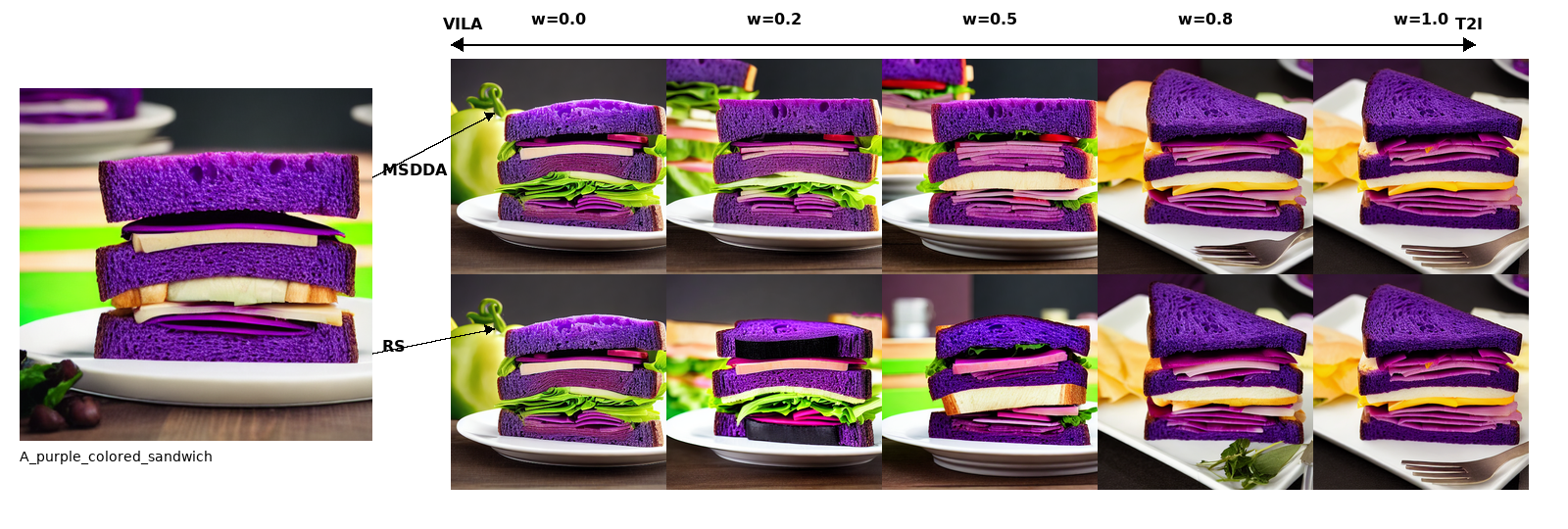}
\includegraphics[width=0.6\textwidth]{A_red_apple_and_a_purple_backpack4.png}
\includegraphics[width=0.6\textwidth]{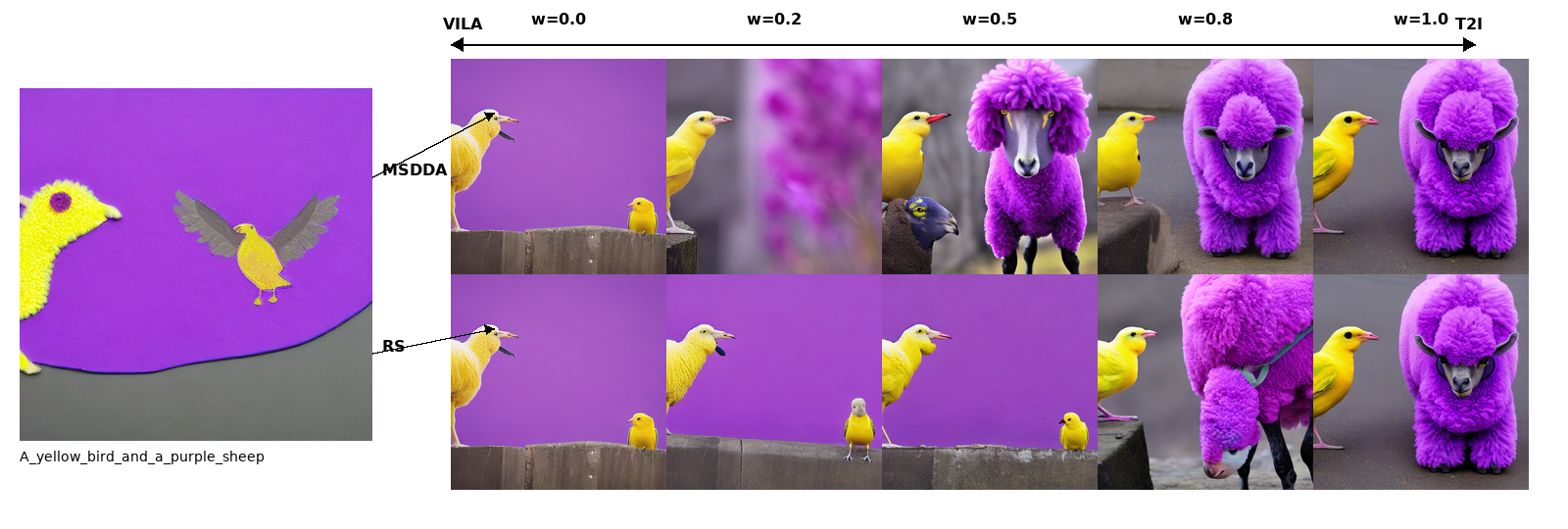}
\includegraphics[width=0.6\textwidth]{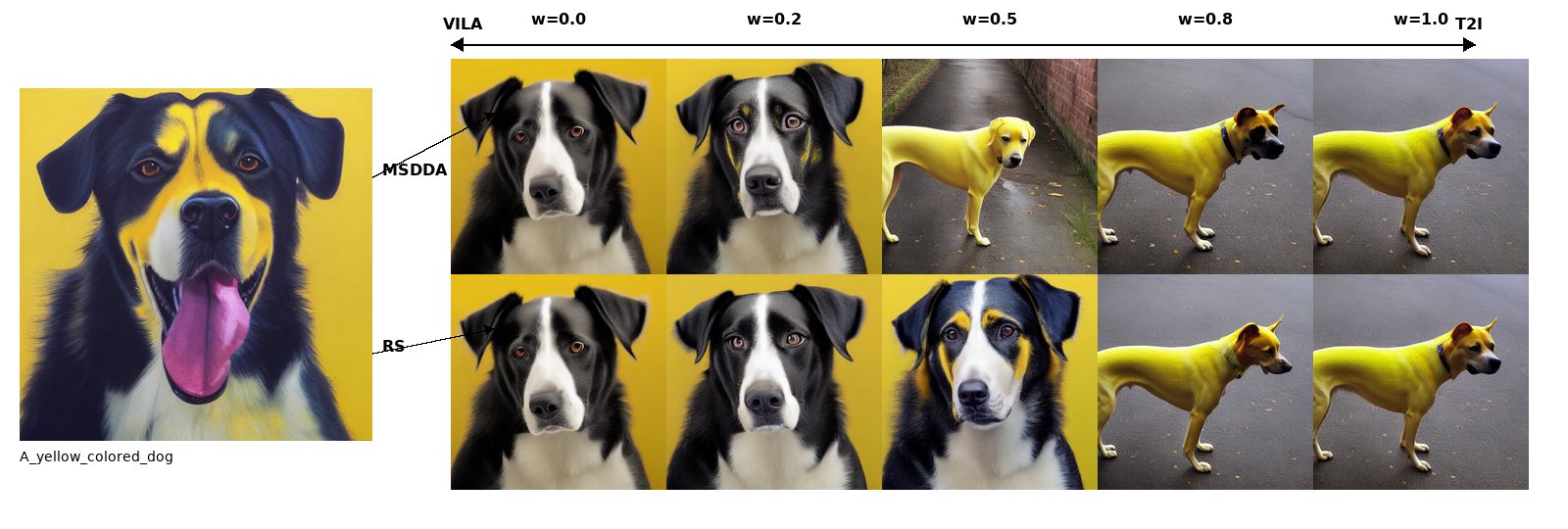}
  \caption{Image samples of our proposed method and RS.}\label{fig:e4}
\end{figure}

\end{document}